\documentclass[11pt]{article}

\usepackage{natbib}
\usepackage{fullpage}

\usepackage{environ}

\usepackage[utf8]{inputenc} %
\usepackage[T1]{fontenc}    %
\usepackage{hyperref}       %
\usepackage{url}            %
\usepackage{booktabs}       %
\usepackage{amsfonts}       %
\usepackage{nicefrac}       %
\usepackage{microtype}      %
\usepackage[ruled,vlined,noend]{algorithm2e}
\usepackage{amsmath,amssymb,amsfonts}
\usepackage{amsthm}
\usepackage{graphicx}
\usepackage{xcolor}
\SetKw{Continue}{continue}
\SetKw{Break}{break}

\newtheorem{prop}{Proposition}

\newtheorem{thm}{Theorem}
\newtheorem{result}{Result}

\newcommand{\tf}{\mathbf}
\newcommand{\eps}{\varepsilon}
\renewcommand{\ker}{\kappa}

\newcommand{\calR}{\mathcal{R}}
\newcommand{\lone}{{\mathcal L_1}}

\newcommand{\R}{\mathbb{R}}
\newcommand{\E}{\mathbb{E}}

\newcommand{\st}{\mathrm{s.t.}}
\newcommand{\reff}{\mathrm{ref}}
\newcommand{\train}{\mathrm{train}}

\newcommand{\tfxw}{\tf{\Phi}_{\tf x\tf w}}
\newcommand{\tfxiw}{\tf{\Phi}_{\pmb \xi\pmb \omega}}
\newcommand{\tfuwxi}{\tf{\Phi}_{\tf u\pmb \omega}}
\newcommand{\tfxin}{\tf{\Phi}_{\pmb \xi\tf n}}
\newcommand{\tfx}{\tf{\Phi}_{\tf x}}
\newcommand{\tfxr}{\tf{\Phi}_{\tf x\tf r}}
\newcommand{\tfxn}{\tf{\Phi}_{\tf x\tf n}}
\newcommand{\tfuw}{\tf{\Phi}_{\tf u\tf w}}
\newcommand{\tfur}{\tf{\Phi}_{\tf u\tf r}}
\newcommand{\tfun}{\tf{\Phi}_{\tf u\tf n}}
\newcommand{\rowtf}{\left\{\tfxr,\tfxn,\tfur,\tfun \right\}}
\newcommand{\rowtfw}{\left\{\tfxw,\tfxn,\tfuw,\tfun \right\}}
\newcommand{\rowtfxi}{\left\{\tfxiw,\tfxin,\tfuwxi,\tfun \right\}}
\newcommand{\fulltf}{\begin{bmatrix} \tfxw &\tfxn  \\ \tfuw &\tfun  \end{bmatrix}}
\newcommand{\fulltfxi}{\begin{bmatrix} \tfxiw &\tfxin  \\ \tfuwxi &\tfun  \end{bmatrix}}

\newcommand{\norm}[1]{\lVert #1 \rVert}
\newcommand{\bignorm}[1]{\left\lVert #1 \right\rVert}

\newcounter{relctr} %
\everydisplay\expandafter{\the\everydisplay\setcounter{relctr}{0}} %

\newcommand\labelrel[2]{%
  \begingroup
    \refstepcounter{relctr}%
    \stackrel{\textnormal{(\alph{relctr})}}{\mathstrut{#1}}%
    \originallabel{#2}%
  \endgroup
}
\AtBeginDocument{\let\originallabel\label} %
 \renewenvironment{proof}[1][\proofname]{{\noindent\bfseries #1. }}{\hfill$\square$}

\theoremstyle{definition}
\newtheorem{example}{Example}
\newtheorem{lemma}{Lemma}
\newtheorem{corollary}{Corollary}
\usepackage[labelfont=bf]{subcaption}
\usepackage[labelfont=bf]{caption}

\usepackage{lmodern}

\newif\ifpreprint
\preprinttrue

\title{Certainty Equivalent Perception-Based Control}

\author{
  Sarah Dean and Benjamin Recht \\
  Department of EECS, University of California, Berkeley
}
\date{}

\begin{document}

\maketitle

\begin{abstract}
In order to certify performance and safety, feedback control requires precise characterization of sensor errors.
In this paper, we provide guarantees on such feedback systems when sensors are characterized by solving a supervised learning problem.
We show a uniform error bound on nonparametric kernel regression under a dynamically-achievable dense sampling scheme.
This allows for a finite-time convergence rate on the sub-optimality of using the regressor in closed-loop for waypoint tracking.
We demonstrate our results in simulation with simplified unmanned aerial vehicle and autonomous driving examples.
\end{abstract}

\section{Introduction}
Machine learning provides a promising avenue for incorporating rich sensing modalities into autonomous systems.
However, our coarse understanding of how ML systems fail limits the adoption of data-driven techniques in real-world applications.
In particular, applications involving feedback require that errors do not accumulate and lead to instability.
In this work, we propose and analyze a baseline method for incorporating a learning-enabled component into closed-loop control, providing bounds on the sample complexity of a reference tracking problem.

Much recent work on developing guarantees for learning and control has focused on the case that dynamics are unknown~\citep{dean2017sample,simchowitz2020naive,mania2020active}.
In this work, we consider a setting in which the uncertainty comes from the sensor generating observations about system state.
By considering unmodeled, nonlinear, and  potentially high dimensional  sensors, we capture phenomena relevant to modern autonomous systems~\citep{codevilla2018end,lambert2018deep,tang2018aggressive}.

Our analysis combines contemporary techniques from statistical learning theory and robust control. 
We consider learning an inverse perception map with noisy measurements as supervision, and show that it is \emph{necessary} to quantify the uncertainty pointwise to guarantee robustness.
Such pointwise guarantees are more onerous than  typical mean-error generalization bounds, and require a number of samples scaling exponentially with dimension.
However, many interesting problems in robotics and automation are low dimensional, and we provide a high-probability pointwise error bound on nonparametric regression for such scenarios. 
Under a dynamically feasible dense sampling scheme, we show uniform converge for the learned map.
Finally, we analyze the suboptimality of using the learned component in closed-loop, and demonstrate the utility of our method for reference tracking.
We close with several numerical examples showing that our method is feasible for many problems of interest in autonomy.
Full proofs of the main results, further discussion on controller synthesis, illustrative examples, and experimental details are included in the longer version of this paper~\citep{dean2020full}.

\subsection{Problem Setting}

Motivated by situations in which control is difficult due to sensing, we consider the task of  waypoint tracking for a system with known, linear dynamics and complex, nonlinear observations.
The setting is succinctly described by the motivating optimal control problem:
\begin{align}
\min_{\pi}~~ &
\sup_{\substack{\{x^\reff_k\}\in \calR\\ \|x_0\|_\infty\leq \sigma_0}}
\left\|\begin{bmatrix}Q^{1/2}(x_k-x^\reff_k) \\ R^{1/2}u_k\end{bmatrix}\right\|_\infty \label{eq:ocp_cost} \\
\st~~ & x_{k+1} = Ax_k + Bu_k ,~~ z_k = g(C x_k),~~ u_k = \pi(z_{0:k}, x^\reff_{0:k}),\label{eq:ocp_dynamics_obs_control}
\end{align}
where $x\in\R^n$ is the system state, $u\in\R^m$ is the control input, 
$z\in\R^q$ is the system observation, and we use the shorthand $x_{0:k} = (x_0, \dots, x_k)$. The linear dynamics are specified by $A\in\R^{n\times n}$, $B\in\R^{n\times m}$, and the matrix $C\in\R^{p\times n}$ determines the state subspace that affects the observation, which we will refer to as the \emph{measurement subspace}. 
We assume that $(A,B)$ is controllable and $(A,C)$ is observable.
This optimal control problem seeks a causal policy $\pi$ to minimize a robust waypoint tracking cost. 
The objective in~\eqref{eq:ocp_cost} penalizes the maximum deviation from any reference trajectory in the class $\calR$ as well as the maximum control input, with the relative importance of these terms determined by $Q\in\R^{n\times n}$ and $R\in\R^{m\times m}$.
A common class of reference signal is those with bounded differences, i.e. $\|x_k^\reff - x_{k+1}^\reff\|_\infty \leq \Delta$.

As defined by the constraints of this problem, the system has linear dynamics, but nonlinear observations~\eqref{eq:ocp_dynamics_obs_control}.
We suppose that the parameters of the linear dynamics $(A,B,C)$ are known, while the observation function $g:\R^p\to\R^q$ is unknown.
This emulates a natural setting in which the dynamics of a system are well-specified from physical principles (e.g. vehicle dynamics), while observations from the system (e.g. camera images) are complex but encode a subset of the state (e.g. position).
We assume that $g$ is continuous and that there is a continuous inverse function $h:\R^q\to \R^p$ with $h(g(y)) = y$.
For the example of a dashboard mounted camera, such an inverse exists whenever each camera pose corresponds to a unique image, which is a reasonable assumption in sufficiently feature rich environments.

If the inverse function $h$ were known, then the optimal control problem is equivalent to a linear output-feedback control
problem. 
\ifpreprint
We expand further on this point in Section~\ref{app:reference_tracking}.
\else
\fi 
We therefore pose a learning problem focused on the unknown inverse function $h$, which we will call a \emph{perception map}.
We suppose that during the training phase, there is an additional system output, 
\begin{align}\label{eq:noisy_sensor}
y_k^\train = Cx_k + \eta_k
\end{align}
where each $\eta_k$ is zero mean and has independent entries bounded by $\sigma_\eta$.
This assumption corresponds to using a simple but noisy sensor for characterizing the complex sensor.
The noisy system output will both supervise the learning problem and allow for the execution of a sampling scheme where the system is driven to sample particular parts of the state space.
Notice that due to its noisiness, using this sensor for waypoint tracking would be suboptimal compared using transformed observations. %

We will show that the \emph{certainty equivalent} controller, i.e. the controller which treats the approximation $\widehat h(z) \approx Cx$ as true, achieves a cost which converges towards optimality.
\begin{result}[Informal]
Certainty equivalent control using a perception map learned with $T$ sampled data points by our method achieves a suboptimality bounded by
\begin{align}
c(\widehat\pi) - c(\pi_\star)\lesssim L r_\star S_\star \left( \frac{\sigma}{T}\right)^{\frac{1}{p+4}}\:.
\end{align}
Here, $c(\cdot)$ is the cost objective in~\eqref{eq:ocp_cost}, $L$ describes the continuity of the relationship between system state and the observations, $r_\star$ bounds the region of system operation under the optimal control law, $S_\star$ is the sensitivity to measurement error of the optimal closed-loop, $\sigma$ is proportional to the sensor noise in~\eqref{eq:noisy_sensor}, and $p$ is the dimension of measurement subspace.
\end{result}

\subsection{Related Work}

The model we study is inspired by examples of control from pixels, primarily in the domain of robotics.
At one end of the spectrum, calibrated camera characteristics are incorporated directly into physical models; this type of visual-inertial odometry enables agressive control maneuvers~\citep{tang2018aggressive}.
On the other end are policies that map pixels directly to low-level control inputs, for example via imitation learning~\citep{codevilla2018end}.
Our model falls between these two approaches: relying on known system dynamics, but deferring sensor characterization to data.

The observation of a linear system through a static nonlinearity is classically studied as a \emph{Weiner system} model~\citep{schoukens2016identification}.
While there are identification results for Weiner systems,
they apply only to single-input-single-output systems, and often require
assumptions that do not apply to the motivation of cameras~\citep{hasiewicz1987identification,wigren1994convergence,greblicki1997nonparametric,tiels2014wiener}.
More flexible identification schemes have been proposed~\citep{lacy2001subspace,salhi2016combined}, but
they lack theoretical guarantees.
Furthermore, these approaches focus on identifying the full forward model, which may not be necessary for good closed-loop performance.
The variant of Weiner systems that we study is closely related to
our recent work, which focused on robustness in control design
~\citep{dean2019robust}.
We now extend these ideas to directly consider issues of sampling and noise.

There is much related and recent work at the intersection of learning and control.
Similar in spirit is a line of work on the Linear Quadratic Regulator which focuses on issues of system identification and sub-optimality~\citep{dean2017sample,dean2018regret,abbasi2011regret}.
This style of sample complexity analysis has allowed for illuminating comparisons between model-based and model-free policy learning approaches~\citep{tu2018gap,abbasi2019model}.
\citet{mania2019certainty,simchowitz2020naive} show that the simple strategy of model-based certainty equivalent control is efficient, though the argument is specialized to linear dynamics and quadratic cost.
For nonlinear systems, analyses of learning often focus on ensuring safety over identification or sub-optimality~\citep{taylor2019episodic,berkenkamp2017safe,wabersich2018linear,cheng2019end}, and rely on underlying smoothness for their guarantees~\citep{liu2019robust,nakka2020chance}.
An exception is a recent result by~\citet{mania2020active}
which presents finite sample guarantees for parametric nonlinear system identification.

While the majority of work on learning for control focuses on settings with full state observation, output feedback is receiving growing attention for linear systems~\citep{simchowitz2019learning,simchowitz2020improper} and for safety-critical systems~\citep{laine2020eyes}. Recent work in closely related problem settings includes~\citet{mhammedi2020learning}, who develop sample complexity guarantees for LQR with nonlinear observations and~\cite{misra2020kinematic}, who leverage representation learning in Block MDPs; however, neither address issues of stability due to focusing on finite horizon problems.

The statistical analysis presented here focuses on nonparametric pointwise error bounds over a compact set.
Distinct from mean-error generalization arguments most common in learning theory,
our analysis is directly related to classical statistical results on uniform convergence~\citep{devroye1978uniform,liero1989strong,hansen2008uniform}.
Our motivation is related to conformal regression~\citep{lei2014distribution,barber2019limits}, which relies on exchangeability assumptions that can be adapted to data from dynamical systems with mixing arguments~\citep{chernozhukov2018exact}.
\section{Uniform Convergence of Perception}

In this section, we introduce a sampling scheme and nonparametric regression strategy for learning the predictor $\widehat h$ and show that the resulting errors are uniformly bounded.
\ifpreprint
While it is typical in the machine learning community to consider mean-error generalization bounds,
we show in the following example that it is necessary to quantify the uncertainty pointwise to guarantee robustness.
This motivating example shows how errors within sets of vanishingly small measure can cause systems to exit bounded regions of well-characterized perception and lead to instability.
For a simple setting, we construct a reference trajectory that leads to instability, illustrating that feedback control objectives require very strong function approximation error bounds.

\begin{example}
Consider a one-dimensional linear system that is open-loop unstable ($|a|>1$) with
an arbitrary linear reference tracking controller: 
\[x_{k+1} = a x_k + u_k,\quad 
u_k = \sum_{t=0}^k K_t^\reff x^\reff_{k-t} +K_t^x \widehat x_{k-t} .\]
Further suppose that the perception map has been characterized within the bounded interval $[-r,r]$ and is imperfect only around the point $\bar x\in[-r,r]$:
\ifpreprint
\[\widehat x = \widehat h(g(x)) =\begin{cases}  0 &|x|\geq r~~\text{or}~~ x=\bar x\\
x&\text{otherwise}\end{cases}\]
\else
\[\widehat x = \widehat h(g(x)) =  0 ~~\text{if}~~ |x|\geq r~~\text{or}~~ x=\bar x ~~\text{and}~~ x~~\text{otherwise}.\]
\fi
Then consider a trajectory starting at $x_0=0$ with $x^\reff_0 = \frac{\bar x}{K_0^r}$, $x^\reff_1 = \frac{1}{K_0^r}\big(r-(a+\tfrac{K_1^r}{K_0^r})\bar x\big)$, and
$x^\reff_k=0$ for all $k\geq 2$.
Despite an initial condition and reference trajectory contained within the well-characterized interval, the perception error at $\bar x$ causes  the unstable system 
trajectory
\[|x_k| = |a^{k-1}\bar x| \to \infty~\text{as}~k\to\infty\:.\]
\end{example}

Motivated by this example, we begin by introducing a method for nonparametric regression and present a data-dependent pointwise error bound.
Then, we propose a dense sampling scheme and show a uniform error bound over the sampled region of the measurement subspace.
\fi

\subsection{Nonparametric Regression}\label{sec:regression}

Since uniform error bounds are necessary for robust guarantees, we now introduce a method to learn perception maps with such bounds.
For simplicity of analysis and exposition, we focus on Nadarya-Watson estimators.
We expect that our insights will generalize to more complex techniques,
and we demonstrate similarities with additional regressors in simulation experiments presented in Section~\ref{sec:experiments}.

The Nadarya-Watson regression estimators with training data $\{(z_t, y_t^\train)\}_{t=0}^T$, bandwidth $\gamma \in\R_+$, and metric $\rho:\R^q\times\R^q\to\R_+$ have the form
\begin{align}
\begin{split}\label{eq:regression_estimator}
\widehat h(z) = \sum_{t=0}^T \frac{\ker_\gamma(z_t,z)}{s_T(z)} y^\train_t ,\quad s_T(z) = \sum_{t=0}^T \ker_\gamma\left(z_t,z\right),\quad
\ker_\gamma\left(z_t,z\right) = \ker\left(\tfrac{\rho(z_t,z)}{\gamma}\right) \:,
\end{split}
\end{align}
with $\widehat h (z) = 0$ when $s_T(z)=0$ and $\ker:\R_+\to[0,1]$ is a kernel function.
We assume that the kernel function is Lipschitz with parameter $L_\ker$ and that $\ker(u)=0$ for $u>1$, and define the quantity
$V_\ker = \int_{ \R^p_+}  \ker\left(\|y\|_\infty\right) d y$. 
Thus, predictions are made by computing a weighted average over the labels $y_t^\train$ of training datapoints whose corresponding observations $z_t$ are close to the current observation, as measured by the metric $\rho$.
We assume the functions $g$ and $h$ are Lipschitz continuous with respect to $\rho$, i.e. for some $L_g$ and $L_h$
\begin{align}\label{eq:smooth_assumption}
\rho(g(y), g(y')) \leq L_g \|y-y'\|_\infty,\quad
\|h(z) - h(z')\|_\infty \leq L_h\rho(z, z')\:. %
\end{align}
While our final sub-optimality results depend on $L_g$ and $L_h$, the perception map and synthesized controller do not need direct knowledge of these parameters.

For an arbitrary $z$ with $s_T(z)\neq 0$, the prediction error can be decomposed as
\begin{align}
\|\widehat h(z) - h(z)\|_\infty
&\leq \Big\| \sum_{t=0}^T \frac{\kappa_\gamma(z_t,z)}{s_T(z)}(Cx_t - Cx) \Big \|_\infty + \Big\| \sum_{t=0}^T \frac{\kappa_\gamma(z_t,z)}{s_T(z)} \eta_t \Big\|_\infty\:.
\end{align}
The first term is the approximation error due to finite sampling, even in the absence of noisy labels.
This term can be bounded using the continuity of the true perception map $h$.
The second term is the error due to measurement noise.
We use this decomposition to state a pointwise error bound, which can be used to provide tight data-dependent estimates on error.

\begin{lemma}\label{lem:regression_error}
For a learned perception map of the form~\eqref{eq:regression_estimator} with training data as in~\eqref{eq:noisy_sensor} collected during closed-loop operation of a system satisfying~\eqref{eq:smooth_assumption}, we have with probability at least $1-\delta$ that for a fixed $z$ with $s_T(z)\neq 0$,
\begin{align}
\|\widehat h(z) - h(z)\|_\infty &\leq \gamma L_h +  \frac{\sigma_\eta}{\sqrt{s_T(z)}} \sqrt{\log\left(p^2\sqrt{s_T(z)}/\delta\right)}\:.
\end{align}
\end{lemma}

\ifpreprint
We present the proof of this result in the Appendix~\ref{app:proofs}.
\else
\fi
The expression illustrates that there is a tension between having a small bandwidth $\gamma$ and ensuring that the coverage term $s_T(z)$ is large.
Notice that most of the quantities in this upper bound can be readily computed from the training data; only $L_h$, which quantifies the continuity of the map from observation to state, is difficult to measure.
We remark that while useful for building intuition, the result in Lemma~\ref{lem:regression_error} is only directly applicable for bounding error at a finite number of points. 
\ifpreprint
Since our main results handle stability over infinite horizons, they rely on a modified bound introduced in Appendix~\ref{app:proofs} which is closely tied to continuity properties of the estimated perception map $\widehat h$ and the sampling scheme we propose in the next section.
\fi

\subsection{Dense Sampling}\label{sec:dense_sampling}

We now propose a method for collecting training data and show a uniform, sample-independent bound on perception errors under the proposed scheme.
This strategy relies on the structure imposed by the continuous and bijective map $g$, which ensures that driving the system along a dense trajectory in the measurement subspace corresponds to collecting dense samples from the space of observations.
In what follows, we provide a method for driving the system along such a trajectory.

We assume that during training, the system state can be reset according to a distribution $\mathcal{D}_0$ which has has support bounded by $\sigma_0$. 
We do not assume that these states are observed. 
Between resets, an affine control law drives the system to evenly sample the measurement subspace with a combination of a stabilizing output feedback controller and reference tracking inputs:
\begin{align}\label{eq:sampling_controller}
u_t = \sum_{k=0}^t K_k y_{t-k}^\train + u_t^\reff\:.
\end{align}
The stabilizing feedback controller prevents the accumulation of errors resulting from the unobserved reset state.
The closed-loop trajectories resulting from this controller are
\begin{align}\label{eq:state_lti_unrolled}
x_{t} = \Phi_x(t)x_0 + \sum_{k=1}^{t} \Phi_{xu}(k) u^\reff_{t-k}+ \Phi_{xn}(k)  \eta_{t-k}\:,
\end{align}
where the system response variables $\{\Phi_x(k), \Phi_{xu}(k), \Phi_{xn}(k)\}_{k\geq 0}$ arise from the interaction of the stabilizing control law with the linear dynamics; we revisit this fact in detail in Section~\ref{sec:system_response_waypoint}.
As long as the feedback control law $\{K_k\}_{k\geq 0}$ is chosen such that the closed-loop system is stable, the system response variables decay.
Designing such a stabilizing controller is possible since $(A,B)$ is controllable and $(A,C)$ is observable.
We therefore assume that for all $k\geq 0$
\begin{align}\label{eq:M_rho_assumption}
\max\left\{\|C\Phi_{x}(k)\|_\infty, \|C\Phi_{xn}(k)\|_\infty\right\} \leq M\rho^k,%
\end{align}
for some $M\geq 1$ and $0\leq \rho<1$.

\SetAlgoLined
\begin{algorithm}[t]
\SetAlgoLined
\KwIn{System variables $C$, stabilizing controller $\{K_k\}_{k=1}^{n}$ and system response $\{\Phi_{xu}(k)\}_{k=1}^{n}$, sampling radius $\bar r$, target dataset size $T$.}
 \For{$\ell\in\{1,\dots,T\}$}{
 	reset $x_{0,\ell}\sim \mathcal{D}_0$ and sample
 	\ifpreprint
 	 $y^\reff_\ell\sim\mathrm{Unif}(\{y\mid\|y\|_\infty \leq \bar r\})$\;
 	design inputs $\begin{bmatrix} (u_{0,\ell}^\reff)^\top ,\dots,(u_{n-1, \ell}^\reff)^\top \end{bmatrix}^\top := \begin{bmatrix} C\Phi_{xu}(1) &\dots& C\Phi_{xu}(n)\end{bmatrix}^\dagger y^\reff_\ell$\;
	\fi
 	drive the system to states $x_{k+1,\ell}$ with $u_{k,\ell} = \sum_{j=0}^k K_j y^\train_{j-k, t} + u_{k,\ell}^\reff$ for $k=0,\dots,n-1$\;
 }
 \KwOut{Uniformly sampled training data 
 $\{(z_{n,\ell}, y_{n,\ell}^\train)\}_{\ell=1}^{T}=:\{(z_{t}, y_{t}^\train)\}_{t=1}^{T}$ %
 }

 \caption{Uniform Sampling with Resets}\label{alg:ref_inputs}
\end{algorithm}

The reference inputs $u_t^\reff$ are chosen to ensure even sampling.
Since the pair $(A, B)$ is controllable, these reference inputs can drive the system to any state within $n$ steps.
Algorithm~\ref{alg:ref_inputs} leverages this fact to construct control sequences which drive the system to points uniformly sampled from the measurement subspace.
Its use of system resets ensures independent samples; we note that since the closed-loop system is stable, such a ``reset'' can approximately be achieved by waiting long enough with zero control input.

As a result of the unobserved reset states and the noisy sensor, the states visited while executing
Algorithm~\ref{alg:ref_inputs} do not exactly follow the desired uniform distribution. They can be decomposed into two terms:
\[C x_{n,\ell} = {\sum_{k=1}^{n} C \Phi_{xu}(k) u^\reff_{n-k, \ell}} + \Big({C \Phi_x(n)x_{0,t} + \sum_{k=1}^{n} C \Phi_{xn}(k) \eta_{n-k, \ell}}\Big) =: {y_\ell^\reff}+{w_\ell}
\]
where $y^\reff_\ell$ is uniformly sampled from $\{y\mid\|y\|_\infty \leq \bar r\}$, and the noise variable $w_\ell$ is bounded:
\[\|w_\ell\|_\infty \leq \|C \Phi_x(n)\|_\infty \|x_0\|_\infty + \sum_{\ell=1}^{n} \|C \Phi_{xn}(\ell)\|_\infty  \|\eta_{n-\ell}\|_\infty 
\leq \frac{M\max\{\sigma_0,\sigma_\eta\}}{1-\rho}\:.\]
The following Lemma shows that uniform samples corrupted with independent and bounded noise
ensure dense sampling of the measurement subspace by providing a high probability lower bound on the coverage $s_T(z)$.

\begin{lemma}\label{lem:sT_bound_concentration}
Suppose that training data $\{z_t\}_{t=1}^T$ is collected with a stabilizing controller satisfying~\eqref{eq:M_rho_assumption} according to Algorithm~\ref{alg:ref_inputs} with $\bar r \geq r+\frac{M\max\{\sigma_0,\sigma_\eta\}}{1-\rho} + \tfrac{\gamma}{L_g}$ from a system satisfying~\eqref{eq:smooth_assumption}.
 Then for all $z$ observed from a state $x$ satisfying $\|Cx\|_\infty \leq r$,
\begin{align*}
s_T(z) \geq \frac{1}{2}\sqrt{T V_\ker} \left(\frac{\gamma}{\bar r L_g } \right)^{\frac{p}{2}}
\end{align*}
with probability at least $1-\delta$ as long as
$T  \geq 8V_\ker^{-1}\log(1/\delta) ( \bar rL_h L_g^2 )^{p}\gamma^{-p}$.
\end{lemma}

We use this coverage property of the training data and
the error decomposition presented in Section~\ref{sec:regression}  to show our main uniform convergence result.

\begin{thm}\label{thm:uniform_convergence}
If training data satisfying~\eqref{eq:noisy_sensor} is collected with a stabilizing controller satisfying~\eqref{eq:M_rho_assumption} by the sampling scheme in Algorithm~\ref{alg:ref_inputs} with radius $\bar r = \sqrt{2}r$ from
a system satisfying~\eqref{eq:smooth_assumption},
then as long as the system remains within the set $\{x\mid \|Cx\|_\infty\leq r\}$, the Nadarya-Watson regressor~\eqref{eq:regression_estimator}
 will have bounded perception error for every observation $z$:
\begin{align}\label{eq:uniform_convergence}
\|\widehat h(z) - h(z)\|_\infty \leq &\gamma L_h +  \frac{\sigma_\eta}{T^{\frac{1}{4}}}   \left( \frac{L_g \sqrt{2} r}{\gamma} \right)^{\frac{p}{4}}\left(\sqrt{p\log\left(T^2/\delta \right)} +1\right)
\:,
\end{align}
with probability at least $1-\delta$ as long as $\gamma\leq L_g((\sqrt{2}-1)r - {M\max\{\sigma_0,\sigma_\eta\}}({1-\rho})^{-1})$ and

 \[T  \geq \max\left\{ 8 p V_\ker^{-1} (\sqrt{2} L_h L_g^2 )^{p}(r/\gamma)^{p} \log(  T^2/\delta ),~~
 V_\ker^{-\frac{1}{3}} (24   L_\ker L_h)^{\frac{4}{3}} L_g ^{\frac{p}{3}} \left(r/\gamma\right)^{\frac{p+4}{3}}   \right\}\:.\]

\end{thm}\ifpreprint
We present a proof in Appendix~\ref{app:proofs}.
\else
\fi
\section{Closed-Loop Guarantees}\label{sec:control}

The previous section shows that nonparametric regression can be successful for learning the perception map within a bounded region of the state space.
But how do these bounded errors translate into closed-loop performance guarantees, and how can we ensure that states remain within the bounded region?
To answer this question, we recall the waypoint tracking problem in~\eqref{eq:ocp_cost}.

\subsection{Linear Control for Waypoint Tracking}\label{sec:system_response_waypoint}

Consider a linear control law for waypoint tracking
$u_t = \sum_{k=0}^t K^{(y)}_k y_{t-k} + \sum_{k=0}^{t} K^{(r)}_k x^\reff_{t-k}$ which depends on some output signal $y_k = Cx_k +\eta_k$.
Similar in form to the controller used for sampling~\eqref{eq:sampling_controller}, this linear waypoint tracking controller can be viewed as computing $u_t^\reff$ based on waypoints $x^\reff_k$.
As first discussed in our sampling analysis~\eqref{eq:state_lti_unrolled}, the closed-loop behavior of a linear system in feedback with 
a linear controller
can be described as a linear function of noise variables.
\ifpreprint
Example~\ref{ex:static_lti} shows how these system response variables arise from static state estimators and control laws.

\begin{example}\label{ex:static_lti}
Consider the classic Luenberger observer,
\begin{align}
\widehat x_{t+1} = A \widehat x_t + B u_t + L(y_t - C\widehat x_t)\:,
\end{align}
and control inputs chosen as $u_t = K\widehat x_t + u^{\reff}_t$, then the closed-loop system has dynamics
\begin{align*}
\begin{bmatrix}x_{t+1}\\ e_{t+1}\end{bmatrix} = 
\begin{bmatrix}A+BK& BK \\ 0 & A-LC\end{bmatrix} \begin{bmatrix}x_{t}\\ e_{t}\end{bmatrix} + \begin{bmatrix}  0 \\  L\end{bmatrix}  \eta_{t}
+ \begin{bmatrix}B\\ 0\end{bmatrix} u^{\reff}_t,
\end{align*}
where $e_t = \widehat x_t- x_t $.
The magnitude of the eigenvalues of $A+BK$ and $A-LC$ thus determine stability and the decay parameter $\rho$, while the transients determine $M$.
\end{example}
\else
\fi

To facilitate our discussion of the general \emph{system response}, we will introduce boldface notation for infinite horizon signals and convolutional operators.
Under this notation, we can write $u_t = \sum_{k=0}^t K_k x_{t-k}$ equivalently as $\tf u = \tf K \tf x$.
We introduce the signal and linear operator norms
 \[\|{\tf x}\|_\infty = \sup_k \|{x_k}\|_\infty,\quad \norm{\tf \Phi}_\lone = \sup_{\norm{\tf w}_\infty \leq 1} \norm{\tf \Phi \tf w}_\infty\:.\]
These signals and operators can be concretely represented in terms of semi-infinite block-lower-triangular Toeplitz matrices acting on semi-infinite vectors or in the $z$-domain with $\tf x = \sum_{k=0}^\infty x_k z^{-k}$.

Using this representation, a linear controller $(\tf K_y,\tf K_r)$ induces the
system responses
\begin{align}\label{eq:sys_response}
\tfx &=(zI - (A + B \tf K_y C\tf))^{-1} & \tfxn &= \tfx B \tf K_y&   \tfxr &= \tfx B \tf K_r 
\end{align}
which are well defined as long as the interconnection is stable.
Under this definition, the state signal can be written as $\tf x =\tfxr \tf x^\reff + \tfxn \tf n + \tfx x_0$.
The responses $\tfun$ and $\tfur$ can be similarly defined, since $\tf u$ is linear in $\tf x$ and $\tf x^\reff$.

\ifpreprint
In Section~\ref{app:reference_tracking}, we show how the
\fi 
robust waypoint tracking cost from~\eqref{eq:ocp_cost} can be cast as an $\lone$ norm on system response variables.
\ifpreprint
For now,
\fi 
it is only necessary to view system response variables as objects that arise from linear controllers.
However, system response variables can also be used to synthesize linear controllers.
This is a key insight of the \emph{System Level Synthesis} framework, first introduced by~\cite{SysLevelSyn1}.
\ifpreprint
We will return to this idea in Section~\ref{app:reference_tracking} after proving our main result.
\fi %

\subsection{Suboptimality of Certainty-Equivalence}

Suppose that we apply a linear controller to our perception estimates,
\begin{align}\label{eq:CE_controller}
\tf u = \widehat\pi(\tf z, \tf x^\reff) = \tf K_y \widehat h(\tf z) + \tf K_r \tf x^\reff =: \tf K(\widehat h(\tf z), \tf x^\reff ).
\end{align}
This is the certainty equivalent controller, which treats the learned perception map as if it is true.
We will compare this controller with $\pi_\star = \tf K(h(\cdot),\cdot)$, the result of perfect perception. 
The suboptimality depends on the following quantity, which bounds $\|C\tf x\|_\infty$ under the optimal control law
\[r_{\max}(\tf\Phi) = \sup_{\substack{\tf x^\reff\in \calR\\ \|x_0\|_\infty\leq \sigma_0}} \|C\tfxr\tf x^\reff+C\tfx x_0\|_\lone\:.\]
The magnitude of this value depends on properties of the considered reference signal class. 
\ifpreprint
In Section~\ref{app:reference_tracking}, we present the specific form of this quantity for reference signals that are bounded and have bounded differences.
\fi

\begin{prop}\label{prop:subopt}
Let $\rowtf$ denote the system responses induced by the controller $\tf K$ as in~\eqref{eq:sys_response}, and let $c(\pi_\star)$ denote the cost associated with the policy $\pi_\star = \tf K(h(\cdot),\cdot)$.
Then for a perception component with error bounded by $\varepsilon_h$ within the set $\{x\mid\|Cx\|_\infty\leq r\}$,
the sub-optimality of the certainty-equivalent controller~\eqref{eq:CE_controller} is bounded by
\[c(\widehat\pi) - c(\pi_\star)\leq \varepsilon_h\left\| \begin{bmatrix}Q^{1/2} \tfxn  \\ R^{1/2}\tfun \end{bmatrix}  \right\|_\lone\:.\]
as long as the sampled region is large enough and the errors are small enough,
$\varepsilon_h \leq \frac{r - r_{\max}(\tf\Phi)}{\|C\tfxn \|_\lone}$.
\end{prop}
\ifpreprint

\begin{proof}[Proof]
In signal notation, the cost is given by
\[\sup_{\substack{\tf x^\reff\in \calR\\ \|x_0\|_\infty\leq \sigma_0}} 
\left\|\begin{bmatrix}Q^{1/2} (\tf x -\tf x^\reff) \\ R^{1/2}\tf u\end{bmatrix}\right\|_\infty \]
where $\tf x$ and $\tf u$ are the closed-loop trajectories induced by the control policy.

Since both policies use the same linear control law $\tf K$, they induce the same system response $\rowtf$.
The difference comes from the perception signal.
The certainty equivalent controller $\widehat \pi$ assigns inputs as
\[\tf u = \tf K_y \widehat h(\tf z) + \tf K_r \tf x^\reff= \tf K_y C\tf x +  \tf K_y( \underbrace{\widehat h(\tf z) - h(\tf z)}_{\tf n} ) + \tf K_r \tf x^\reff\:.\]
While for the optimal controller, $\tf u = \tf K_y C\tf x + \tf K_r \tf x^\reff$.

Using the system response variables, we can rewrite and upper bound the cost:
\begin{align*}
c(\widehat \pi)
&=
\sup_{\substack{\tf x^\reff\in \calR\\ \|x_0\|_\infty\leq \sigma_0}}  \left\|\begin{bmatrix}Q^{1/2} & \\ & R^{1/2}\end{bmatrix} 
\begin{bmatrix}\tfx &\tfxr-I & \tfxn  \\ \tf K_y C\tfx & \tfur & \tfun \end{bmatrix} 
\begin{bmatrix}x_0 \\ \tf x^\reff\\ \tf n\end{bmatrix}\right\|_\infty\\
&\leq
\sup_{\substack{\tf x^\reff\in \calR\\ \|x_0\|_\infty\leq \sigma_0}} \left\|\begin{bmatrix}Q^{1/2} & \\ & R^{1/2}\end{bmatrix} \begin{bmatrix}\tfx & \tfxr-I   \\ \tf K_y C\tfx & \tfur \end{bmatrix} \begin{bmatrix}  x_0 \\ \tf x^\reff\end{bmatrix}\right\|_\infty
+\left\|\begin{bmatrix}Q^{1/2} & \\ & R^{1/2}\end{bmatrix} \begin{bmatrix} \tfxn  \\ \tfun \end{bmatrix}  \tf n\right\|_\infty\\
&=
c(\pi_\star)+\sup_{\substack{\tf x^\reff\in \calR\\ \|x_0\|_\infty\leq \sigma_0}} \left\|\begin{bmatrix}Q^{1/2} & \\ & R^{1/2}\end{bmatrix} \begin{bmatrix} \tfxn  \\ \tfun \end{bmatrix}  \tf n\right\|_\infty\\
&\leq
c(\pi_\star)+\left\|\begin{bmatrix}Q^{1/2} & \\ & R^{1/2}\end{bmatrix} \begin{bmatrix} \tfxn  \\ \tfun \end{bmatrix}  \right\|_\lone \cdot \sup_{\substack{\tf x^\reff\in \calR\\ \|x_0\|_\infty\leq \sigma_0}} \| \widehat h(\tf z) - h(\tf z)\|_\infty
\end{align*}

Recall the uniform error bound on the learned perception map.
As long as $\|C\tf x\|_\infty \leq r$, we can guarantee that $\| \widehat h(\tf z) - h(\tf z)\|_\infty \leq \varepsilon_h$.
Notice that we have 
\[\|C\tf x\|_\infty \leq \sup_{\substack{\tf x^\reff\in \calR\\ \|x_0\|_\infty\leq \sigma_0}} \|C\tfxr\tf x^\reff + C\tfx x_0 + C\tfxn \tf n \|_\lone\leq r_{\max}(\tf\Phi) + \varepsilon_h \|C\tfxn \|_\lone\:.\]
Therefore, the result follows as long as $ \varepsilon_h  \leq \frac{r-r_{\max}(\tf\Phi)}{\|C\tfxn \|_\lone}$.
\end{proof}
\else
\fi

Thus, the optimal closed-loop system's sensitivity to measurement noise is closely related to the sub-optimality.
\ifpreprint
In the next section, we show how to augment the cost of the waypoint tracking problem in~\eqref{eq:ocp_cost} to make it more robust to perception errors.
\fi

\ifpreprint
We now state our main result.
The proof is presented in the appendix.
\fi
\begin{corollary}\label{coro:main_result}
Suppose that training data satisfying~\eqref{eq:noisy_sensor} is collected
is collected with a stabilizing controller satisfying~\eqref{eq:M_rho_assumption} 
according in Algorithm~\ref{alg:ref_inputs} with $\bar r = 2 r_{\max}(\tf\Phi) \geq \max\{1, M\frac{\max\{\sigma_0,\sigma_\eta\}}{1-\rho}\}$ from
a system satisfying~\eqref{eq:smooth_assumption}, and that the Nadarya-Watson regressor~\eqref{eq:regression_estimator} uses bandwidth $\gamma$ chosen to minimize the upper bound in~\eqref{eq:uniform_convergence}. 
Then the overall suboptimality of the certainty equivalent controller~\eqref{eq:CE_controller} is bounded by
\begin{align}
c(\widehat\pi) - c(\pi_\star)\leq 
  4L_gL_h r_{\max}(\tf\Phi)  \left(\frac{4p^2\sigma_\eta^4 }{T}\right)^{\frac{1}{p+4}} \left\| \begin{bmatrix}Q^{1/2} \tfxn  \\ R^{1/2}\tfun \end{bmatrix}  \right\|_\lone \sqrt{\log(T^2/\delta)}
\end{align}
with probability greater than $1-\delta$ for large enough
$T\geq 4p^2\sigma_\eta^4 \left(10L_gL_h\|C\tfxn \|_\lone\sqrt{\log(T^2/\delta)}\right)^{p+4 }$.

\end{corollary} %
\section{Synthesis of Reference Tracking Controllers} \label{app:reference_tracking}

This section outlines synthesis strategies for output-feedback linear optimal control problems.
We begin with a generic overview of \emph{system level synthesis}, and then explain how to apply the framework to the reference tracking problem that we consider.
Finally, we extend our result in Proposition~\ref{prop:subopt} to propose a robust synthesis problem for perception-based control.

\subsection{System Level Synthesis}

Consider an optimal control problem of the form
\begin{align}
\begin{split}
\min_{\tf K} \quad& \sup_{\substack{\| w_k\|_\infty\leq 1\\ \| n_k\|_\infty\leq 1}}\left\|\begin{bmatrix} Q^{1/2} x_k\\ R^{1/2} u_k \end{bmatrix}\right\|_\infty\\
\st \quad& x_{k+1} = Ax_k + Bu_k + Hw_k\\
& y_k = Cx_k + Nn_k\\
& u_k = \tf K(y_{0:k}),
\end{split}
\label{eq:of-sys}
\end{align}
which seeks to find a controller that minimizes the maximum state and input for a linear system driven by bounded process and measurement noise.
This optimal control problem is an instance of $\lone$ robust control~\citep{dahleh19871}.
In particular, if the cost matrices are defined such that $Q = \mathrm{diag}({1}/{q^2_{i}})$ and $R = \mathrm{diag}({1}/{r^2_{i}})$, then an optimal cost less than 1 guarantees that $|x_{i,k}| \leq q_i$ and $|u_{i,k}| \leq r_i$ for all possible realizations of process and measurement noise.

As discussed in Section~\ref{sec:dense_sampling} and Section~\ref{sec:system_response_waypoint},
any linear system in feedback with a linear controller can be written in terms of the system responses $\{\Phi_{xw}(k), \Phi_{xn}(k), \Phi_{uw}(k), \Phi_{un}(k)\}$,
\begin{align}\label{eq:phis}
\begin{bmatrix} x_k \\ u_k \end{bmatrix} = \sum_{t=1}^{k} \begin{bmatrix} \Phi_{xw}(t) & \Phi_{xn}(t) \\ \Phi_{uw}(t) & \Phi_{un}(t)\end{bmatrix} \begin{bmatrix} Hw_{k-t} \\ N n_{k-t}\end{bmatrix} \iff 
\begin{bmatrix} \tf x \\ \tf u \end{bmatrix} = \begin{bmatrix} \tfxw & \tfxn \\ \tfuw & \tfun\end{bmatrix} \begin{bmatrix} H\tf w \\ N\tf n\end{bmatrix}
\end{align}
where we will use boldface notation for infinite horizon signals and convolutional operators. 
While we previously considered how linear controllers give rise to system response elements, we will now consider how system response elements can be used to synthesize controllers.
Equation~\eqref{eq:phis} is linear in the system response elements, so convex functions of state and input translate to convex functions of the system responses. 
\citet{SysLevelSyn1} show that for any elements $\rowtfw$ constrained to obey
\begin{align} \label{eq:sls_constraints}
\begin{bmatrix} zI-A & -B \end{bmatrix} \fulltf = \begin{bmatrix} I & 0  \end{bmatrix} , ~~ \fulltf \begin{bmatrix} zI-A \\ -C \end{bmatrix}  = \begin{bmatrix} I \\ 0  \end{bmatrix} %
\end{align}
there exists a feedback controller that achieves the desired system responses~\eqref{eq:phis}.
Therefore, any optimal control problem over linear systems can be equivalently written as an optimization over system response elements. 
Reformulating the optimal control problem~\eqref{eq:of-sys} yields
\begin{align}
\begin{split}
\min \quad& \left\|\begin{bmatrix} Q^{1/2} & 0\\ 0 & R^{1/2}  \end{bmatrix} \fulltf \begin{bmatrix} H & 0\\ 0 & N \end{bmatrix} \right\|_\lone\quad\st \quad \eqref{eq:sls_constraints},
\end{split}
\label{eq:of-sys-sls}
\end{align}
which is a convex, albeit infinite-dimensional problem.
In practice, finite-dimensional approximations are necessary, and we briefly describe two that have been successfully applied. 
The first optimizes over only finite impulse response systems by enforcing that $\Phi(t) = 0$ for all $t>T$ and some appropriately large $T$.
This is always possible for systems that are controllable and observable.
The second approach instead enforces bounds on the norm of $\Phi(t)$ for $t>T$ and uses robustness arguments to show that the sub-optimality of this finite dimensional approximation decays exponentially in the approximation horizon $T$ \citep{dean2017sample,matni2017scalable,boczar2018finite,dean2019safely,anderson2019}.

The synthesis~\eqref{eq:of-sys-sls} gives rise to a control law of the form $\tf u = \tf K \tf y$ where $\tf K = \tfun - \tfuw \tfxw^{-1} \tfxn$. This controller can be implemented via a state-space realization \citep{anderson2017structured} or as an interconnection of the system response elements \citep{SysLevelSyn1}.

\subsection{Reference Tracking Controllers}

We now specialize to the linear reference tracking optimal control problem
\begin{align}
\min_{\pi}~~ &
\sup_{\substack{\{x^\reff_k\}\in \calR\\ \|w_k\|_\infty\leq \sigma_w\\\|\eta_k\|_\infty\leq \sigma_\eta}}
\left\|\begin{bmatrix}Q^{1/2}(x_k-x^\reff_k) \\ R^{1/2}u_k\end{bmatrix}\right\|_\infty \label{eq:ocp_cost_app} \\
\st~~ & x_{k+1} = Ax_k + Bu_k +w_k,~~ y_k = C x_k +\eta_k,~~ u_k = \pi(y_{0:k}, x^\reff_{0:k}),\label{eq:ocp_dynamics_obs_control_app}
\end{align}

Note that in comparison to~\eqref{eq:ocp_cost}, we generically consider process noise $w_k$ and measurement noise $\eta_k$.
The process noise can encode an initial condition by defining $w_{-1}=x_0$.
In what follows, we assume that signals in the reference class $\calR$ satisfy boundedness, i.e. $\|x_k^\reff\| \leq r_{\max}$, and bounded differences, i.e. $\|x_k^\reff - x_{k+1}^\reff\| \leq \Delta$.

To use the system level synthesis framework, we reformulate this problem into a disturbance rejection problem.
Because reference signals in the class $\calR$ have bounded differences,
we define the augmented state $\xi_k = [x_k-r_k; r_k]$,  %
disturbances
$\omega_k=[w_k/\sigma_w; (r_{k+1}-r_{k})/\Delta]$, and output $\bar y_k = [Cx_k; r_k]$. 
Under this change of variables, the optimal control problem can be written as
\begin{equation}
\begin{array}{rl}
\displaystyle\min_{\tf u = \tf K\bar{\tf y}}&
\sup_{\substack{\|\omega_k\|_\infty\leq 1\\\|\eta_k\|_\infty\leq \sigma_\eta}}
\bignorm{\begin{matrix}\bar Q^{1/2}\xi_k \\ R^{1/2}u_k\end{matrix}}_\infty\\
\text{subject to} & \xi_{k+1} = \bar A \xi_k + \bar B u_k + \bar H w_k\:, ~~\bar y_k = \bar C \xi_k + N\eta_k\\
\end{array}
\label{eq:l1-control}
\end{equation}
where $\bar Q^{1/2}= \begin{bmatrix}Q^{1/2} & 0\end{bmatrix}$ and
	\begin{align*}
	\bar A = \begin{bmatrix}A &A-I \\0& I\end{bmatrix},~~ \bar B = \begin{bmatrix} B\\ 0\end{bmatrix},~~\bar C = \begin{bmatrix} C & C \\ 0 & I\end{bmatrix},~~ \bar H = \begin{bmatrix} \sigma_w I & -\Delta I \\  0 & \Delta I\end{bmatrix}, ~~
	N = \begin{bmatrix}I\\0\end{bmatrix}\:.
	\end{align*}
Then as a problem over system response variables $\rowtfxi$:
\begin{align}
\begin{split}
\min \quad& \left\|\begin{bmatrix} \bar Q^{1/2} & 0\\ 0 & R^{1/2}  \end{bmatrix} \fulltfxi \begin{bmatrix} \bar H & 0\\ 0 & \sigma_\eta N \end{bmatrix} \right\|_\lone\quad\st \quad \eqref{eq:sls_constraints},
\end{split}
\label{eq:of-sys-sls-tracking}
\end{align}
The resulting linear controller will be $\tf u = \tf K\bar{\tf y} = \tf K_y \tf y + \tf K_r \tf r $, where $\tf K = \tfun - \tfuwxi \tfxiw^{-1} \tfxin$.
Notice that for the original optimal control problem in~\eqref{eq:ocp_cost}, we would set $\sigma_\eta=0$ and $\sigma_w = \sigma_0$.

\subsection{Robust Synthesis}

Using the argument for Proposition~\ref{prop:subopt}, we can modify the synthesis problem to take into account the limitations of the learned perception map.
We consider a learned perception map satisfying
\[\|\widehat h(z_k) - h(z_k)\|_\infty \leq \eps_h\]
as long as $\|Cx_k\|_\infty \leq r$.
Then by using the perception map as a virtual sensor, $y_k = \widehat h(z_k)$, we have that $y_k = Cx_k + \eta_k$ where $\|\eta_k\|_\infty\leq \eps_h$ as long as $Cx_k$ is bounded.

Then in order to enforce the constraint, 
\begin{align*}
\|C \tf x\|_\infty &\leq\|C\tf r\|_\infty + 
\|C(\tf x-\tf r)\|_\infty \\
&\leq r_{\max} + \|C\begin{bmatrix}I&0\end{bmatrix}\pmb \xi\|_\infty\\
&= r_{\max} + \|C\begin{bmatrix}I&0\end{bmatrix}
(\tfxiw \bar H \pmb \omega + \tfxin N \tf n) \|\\
&\leq \underbrace{ r_{\max} + \|C\begin{bmatrix}I&0\end{bmatrix}
\tfxiw \bar H\|_\lone }_{r_{\max}(\tf\Phi)} + \eps_h\|C\begin{bmatrix}I&0\end{bmatrix} \tfxin N \|_\lone \leq r
\end{align*}
where we make use of the boundedness of reference signals and derive a specific formula for $r_{\max}(\tf\Phi)$.

Therefore, the robust synthesis problem is given by
\begin{align}
\begin{split}
\min \quad& \left\|\begin{bmatrix} \bar Q^{1/2} & 0\\ 0 & R^{1/2}  \end{bmatrix} \fulltfxi \begin{bmatrix} \bar H & 0\\ 0 & \eps_h N \end{bmatrix} \right\|_\lone\\
\st \quad &\eqref{eq:sls_constraints}, ~~  \|C\begin{bmatrix}I&0\end{bmatrix}
\tfxiw \bar H\|_\lone + \eps_h\|C\begin{bmatrix}I&0\end{bmatrix} \tfxin N \|_\lone \leq r - r_{\max} 
\end{split}
\label{eq:of-sys-sls-tracking-rob}
\end{align}
The constraint highlights the necessity of achieving low errors $\eps_h$ within a sufficiently large radius $r$.
In recent work, \cite{dean2019robust} present a detailed discussion of $\lone$ controller synthesis for a closely related problem, including its implementation as a linear program. 
\section{Experiments}\label{sec:experiments}

To illustrate our results and to probe their limits, we perform experiments in two simulated environments: a simplified unmanned aerial vehicle (UAV) with a downward facing camera and an autonomous driving example with a dashboard mounted camera.
All code necessary for reproducing experiments is available at: \url{github.com/modestyachts/certainty_equiv_perception_control}.

\subsection{Setting and Training Data}
For both UAV and car environments, observations are $300\times 400$ pixel RGB images generated using the CARLA simulator~\citep{carla}. 
For the UAV setting, we fix the height and orientation of the camera so that it faces downward from an elevation of 40m.
For the autonomous driving setting, the height is fixed to ground level and the orientation is determined by the car's velocity.
Figure~\ref{fig:example_obs} shows example observations.

\def \basefigwidth{0.49}

\begin{figure}[t]
\centering
\begin{subfigure}[t]{0.24\textwidth}
\caption{\small UAV sampling}\label{fig:uav_sampling}
\centerline{\includegraphics[width=\columnwidth]{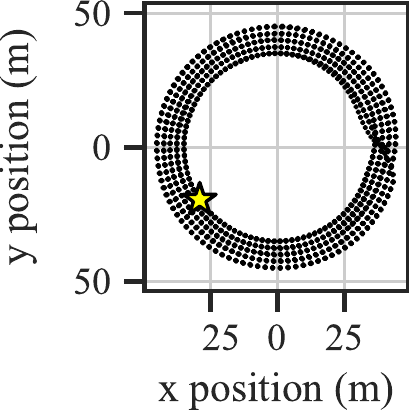}}
\end{subfigure}\qquad
\begin{subfigure}[t]{0.45\textwidth}
\caption{\small Example observations}\label{fig:example_obs}
\centerline{\includegraphics[width=0.49\columnwidth]{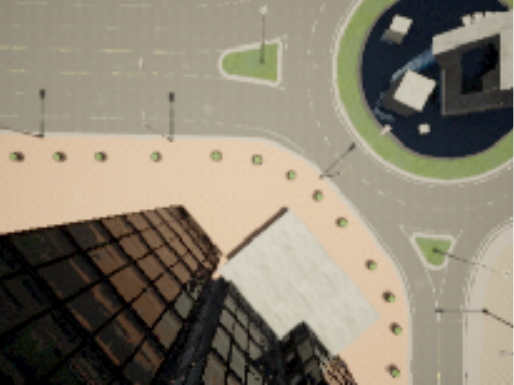}~\includegraphics[width=0.49\columnwidth]{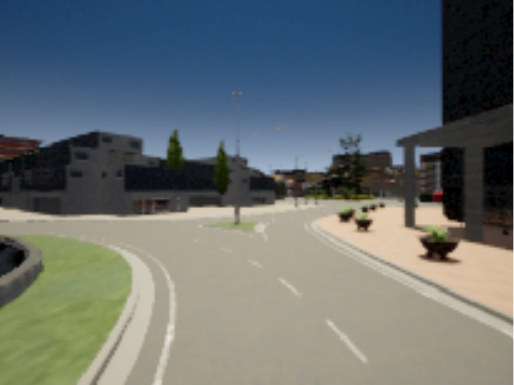}}
\end{subfigure}
\begin{subfigure}[t]{0.24\textwidth}
\caption{\small Car sampling}\label{fig:car_sampling}
\centerline{\includegraphics[width=\columnwidth]{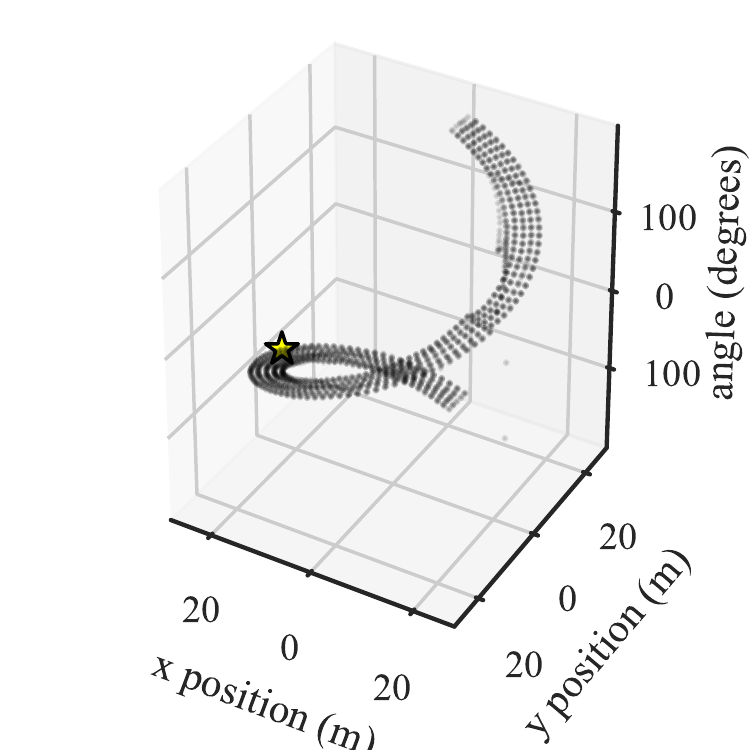}}
\end{subfigure}
\caption{Coverage of training data for (a) UAV and (c) autonomous driving settings. In (b), example observations taken from positions indicated by yellow stars. 
}\label{fig:references}
\end{figure}

In both cases, system dynamics are defined using a hovercraft model, where positions along east-west and north-south axes evolve independently according to double integrator dynamics.
The hovercraft model is specified by 
\begin{align}\label{eq:hovercraft}
A= \begin{bmatrix} 1 & 0.1 & 0 & 0 \\ 0 & 1 & 0&0\\ 0 & 0 &1 & 0.1\\0&0& 0 & 1\end{bmatrix},\quad B=\begin{bmatrix} 0&0\\ 1&0 \\ 0&0\\0&1\end{bmatrix},\quad C= \begin{bmatrix} 1&0 &0&0 \\ 0&0 &1&0\end{bmatrix}\:,
\end{align}
so that $x^{(1)}$ and $x^{(3)}$ respectively represent position along east and north axes, while $x^{(2)}$ and $x^{(4)}$ represent the corresponding velocities.
For the UAV, the rendered image depends only on the position, with a scale factor of 20 to convert into meters in CARLA coordinates.
For the car example, observations are determined as a function of vehicle pose, and thus additionally depend on the heading angle. 
We defined the heading to be $\arctan(x^{(4)}/x^{(2)})$, so the rendered image depends on the position as well as velocity state variables. For the driving setting, the scale factor is 11.

We construct training trajectories by applying a static reference tracking controller to trace circles of varying radius.
For the training phase, the noisy sensor as gives rise to measurements
$y_k^\train = Cx_k + \eta_k$
where $\eta_k$ is generated by clipping a Gaussian with standard deviation $0.01$ between $-1$ and $1$.
We use the static reference tracking controller:
\[u^\reff_k = K(\widehat x_k - C^\dagger y^\reff_k),\]
where the state estimate is computed from $y_{0:k}^\train$ as in Example~\ref{ex:static_lti}.
The feedback parameter $K$ is generated as the optimal LQR controller when costs are specified as $Q=C^\top C$ and $R = I$ while the estimation parameter $L$ is generated as the optimal infinite horizon Kalman filter for process noise with covariance $W=I$ and measurement noise with covariance $V=0.1\cdot I$.
We use the periodic reference
\[y^\reff_k = \begin{bmatrix}a_k\sin(2\pi k/100) & a_k\cos(2\pi k/100)\end{bmatrix}^\top,\quad a_k = 1.75 + 0.125 \left(\lfloor k/100 \rfloor \mod 4\right)\]
which traces circles counterclockwise, ranging in radius from $1.75$ to $2.25$, or 35-45m for the UAV and 19.25-24.75m for the car.
We collect training data for $0\leq k\leq T=2,000$.
Figure~\ref{fig:uav_sampling} and~\ref{fig:car_sampling} display the positions from which training observations and measurements are collected.
Notice that for the car, this strategy results in a sparsely sampled measurement subspace.

\begin{figure}[t]
\centering
\begin{subfigure}[b]{0.75\textwidth}
\caption{\small Perception errors over region}\label{fig:heatmap_errors}
\includegraphics[width=0.95\columnwidth]{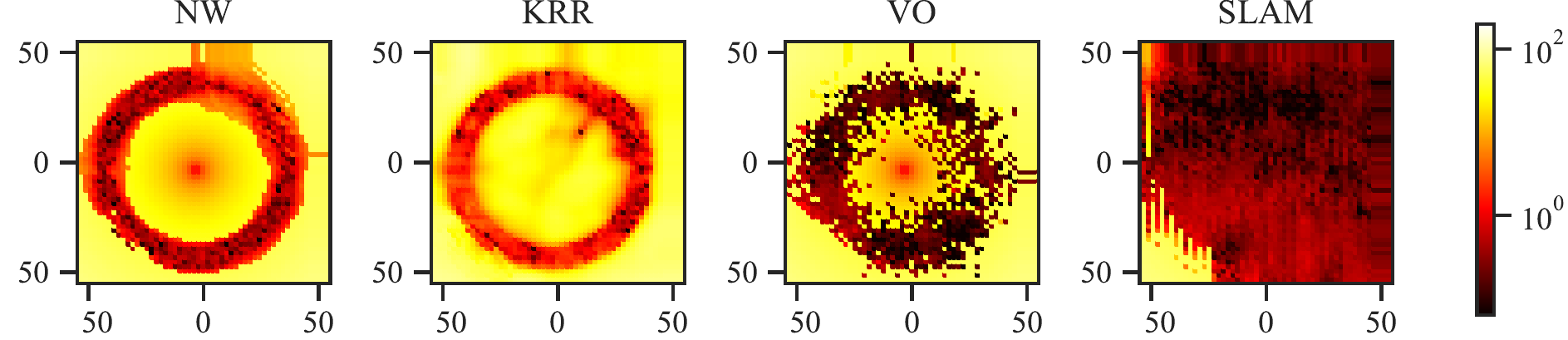}
\end{subfigure}~~~~
\begin{subfigure}[b]{0.25\textwidth}
\caption{\small Perception errors}\label{fig:bar_errors}
\centerline{\includegraphics[width=\columnwidth]{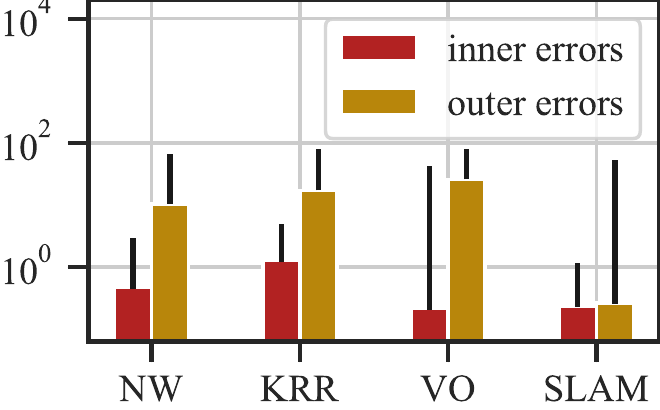}}
\end{subfigure}
\caption{In (a), heatmaps illustrate perception errors. In (b), median and 99th percentile errors within the inner (37-42m radius) and outer (25-55m radius, excluding inner) regions of training data.}\label{figure:errors}
\end{figure}

\subsection{Perception Error Evaluation}

We consider four types of perception predictors:
\begin{itemize}
	\item {\bf Nadarya-Watson (NW):} The estimated position is computed based on training data as in~\eqref{eq:regression_estimator} where $\rho(z, z_t) = \|z-z_t\|_2$ is the $\ell_2$ distance on raw pixels. The only hyperparameter is the bandwidth $\gamma$.

	We investigated a few additional metrics based on $\ell_2$ distance between \emph{transformed pixels}, but did not observe any benefits. The transformations we considered were: 
	pretrained Resnet-18 features,\footnote{\url{https://github.com/christiansafka/img2vec}}
	Guassian blur,\footnote{\url{https://scikit-image.org/docs/stable/api/skimage.filters.html\#skimage.filters.gaussian}}  
	and edge detection-type filters.\footnote{\url{https://scikit-image.org/docs/stable/api/skimage.filters.html}}
	\item {\bf Kernel Ridge Regression (KRR):} The estimated position is computed as
	\[\widehat h(z) = \begin{bmatrix} (y^\train_0)^\top&\dots&(y^\train_T)^\top\end{bmatrix} (\lambda I + K)^{-1} \begin{bmatrix} k(z_0, z) & \dots & k(z_T, z)\end{bmatrix}^\top \]
	where the kernel matrix $K$ is defined from the training data with $K_{ij}=k(z_i, z_j)$ and we use the radial basis function kernel, $k(z,z') = e^{-\alpha\|z-z'\|_2^2}$. The hyperparameters are regularization $\lambda$ and spread $\alpha$.
	\item {\bf Visual Odometry (VO):} This structured method relies on a database of training images with known poses, which we construct using ORB-SLAM~\citep{murTRO2015}, and calibrate to the world frame using position measurements $\{y^\train_t\}$ which determine scale and orientation.
	We use only the first 200 training datapoints to initialize this database.

	New observations $z$ are matched with an image $z_t$ in the training data based on ``bag of words'' image descriptor comparison.
	Then, the camera pose is estimated to minimize the reprojection error between matched features in $z$ and $z_t$, and from this we extract the position estimate $\widehat h(z)$.
	\item {\bf Simultaneous Localization and Mapping (SLAM):} This structured method proceeds very similarly to the VO method described above, with two key differences.
	First, all new observations $z$ are added to the database along with training images.
	Second, pose estimates are initialized based on the previous frame, resulting in a temporal component.
	We implement this method by running ORB-SLAM online.
\end{itemize}
These additional methods allow for comparison with a classical nonparametric approach, a classical computer vision approach, and a non-static state-of-the-art approach.

We evaluate the learned perception maps on a grid of $2,500$ points, with $-2.5\leq x^{(1)}, x^{(3)}\leq 2.5$.
For the car, we set the angle of the grid points as $\arctan(-x^{(1)}/x^{(3)})$ to mimic counter-clockwise driving, which results in a best-case evaluation of the learned perception components.
For SLAM evaluation, the ordering of each evaluated point matters. 
We visit the grid points by scanning vertically, alternating upwards and downwards traversal.
For NW and KRR, we choose hyperparameters which result in low errors within the region covered by training samples. In the UAV setting, we used $\gamma=25,000$, $\alpha=10^{-9}$, and $\lambda=0.111$. 
In the car setting, we used $\gamma=16666.67$.

The resulting errors are displayed for the UAV in
Figure~\ref{fig:heatmap_errors} and for the car in Figure~\ref{fig:heatmap_errors_car}.
The error heatmaps are relatively similar for the three static regressors, with small errors within the training data coverage and larger errors outside of it.
Though VO has very small errors at many points, its heat map looks much noisier. The large errors come from failures of feature matching within a database of key frames from the training data; in contrast, NW and KRR predictions are closely related to $\ell_2$ distance between pixels, leading to smoother errors.
Because SLAM performs mapping online, it can leverage the unlabelled evaluation data to build out good perception away from training samples.
In the UAV setting, SLAM has high errors only due to a tall building obstructing the camera view, visible in Figure~\ref{fig:example_obs}.
Similarly, in the driving setting, a wall and a statue obstruct the view in the locations that SLAM exhibits large errors.
Figure~\ref{fig:bar_errors} and~\ref{fig:bar_errors_car} summarize the evaulations by plotting the median and 99th percentile errors in the \emph{inner region} of training coverage compared with the \emph{outer region}.

\begin{figure}
\centering
\begin{subfigure}[t]{0.7\textwidth}
\caption{\small Perception errors over region}\label{fig:heatmap_errors_car}
\includegraphics[width=0.95\columnwidth]{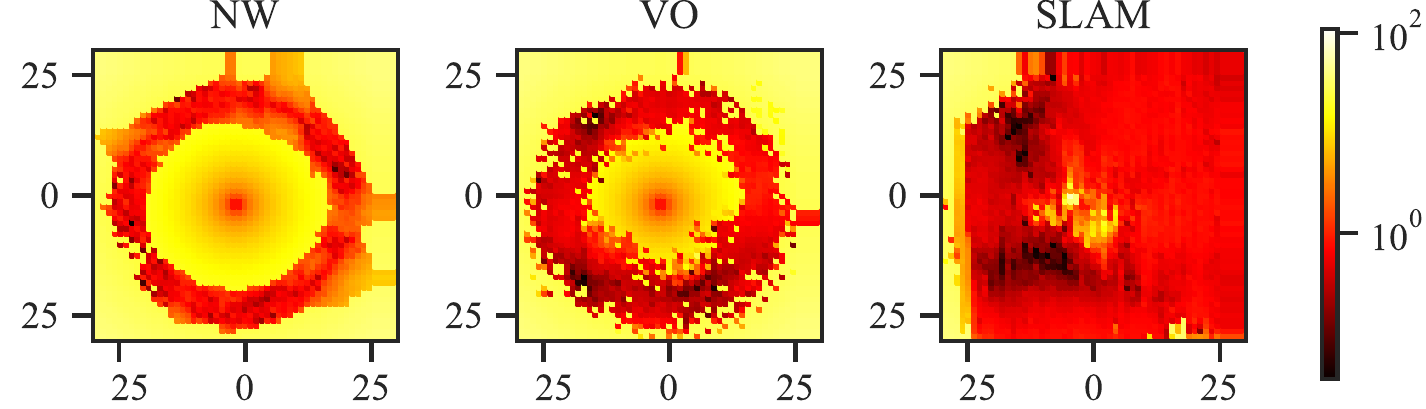}
\end{subfigure}~~~~
\begin{subfigure}[t]{0.3\textwidth}
\caption{\small Perception errors}\label{fig:bar_errors_car}
\centerline{\includegraphics[width=\columnwidth]{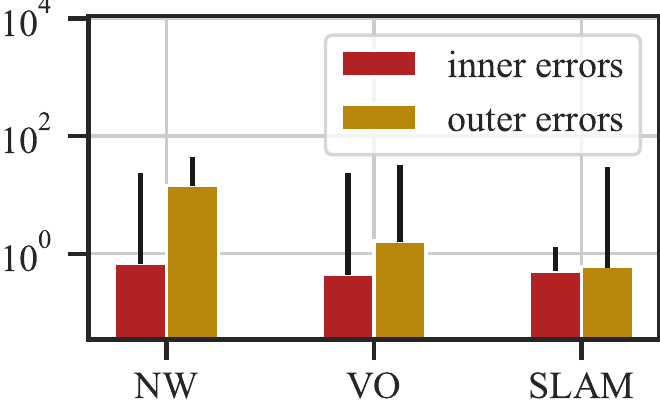}}
\end{subfigure}
\caption{In (a), heatmaps illustrate perception errors. In (b), median and 99th percentile errors within the inner (20.35-23.1m radius) and outer (13.75-30.25m radius, excluding inner) regions of training data.}\label{figure:errors_car}
\end{figure}

\subsection{Closed-Loop Operation}

Finally, we evaluate the closed-loop performance of the predictors by using them for a reference tracking task.
We consider reference trajectories of the form
\begin{align}\label{eq:reference_exp}
y^\reff_k = \begin{bmatrix}a \sin(2\pi k/100) & 2\cos(2\pi k/100)\end{bmatrix}^\top
\end{align}
and a static reference tracking controller which has the same parameters as the one used to collect training data.

We first examine how changes to the reference trajectory lead to failures in Figure~\ref{fig:trajectories_uav}.
For the UAV setting, NW with dense sampling is sufficient to allow for good tracking of the reference with $a=1.9$, but tracking fails when trying to track a reference outside the sampled region with $a=1.6$.
As the reference signal drives the system into a sparsely sampled region, errors increase until eventually the UAV exits the region covered by training data.
Once this occurs, the system loses stability due to the perception failure.
Despite briefly regaining perception, the system does not recover.

The autonomous driving setting illustrates failure modes that arise for nonparametric predictors when training data is not dense.
The trajectories in Figure~\ref{fig:trajectories} show that even though the reference is contained within the well-sampled region ($a=1.9$), NW experiences failures due to neglecting to densely sample with respect to angle.
Errors cause the system to deviate from the reference, and eventually the system exits the region covered by training data, losing stability due to the perception failure.
SLAM does not have this problem.

\begin{figure}[b]
    \centering
	\includegraphics[height=7em]{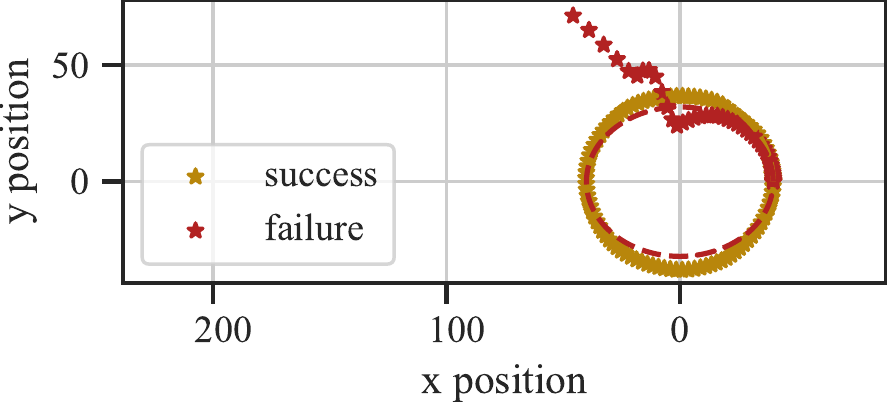}~~~~
	\includegraphics[height=7em]{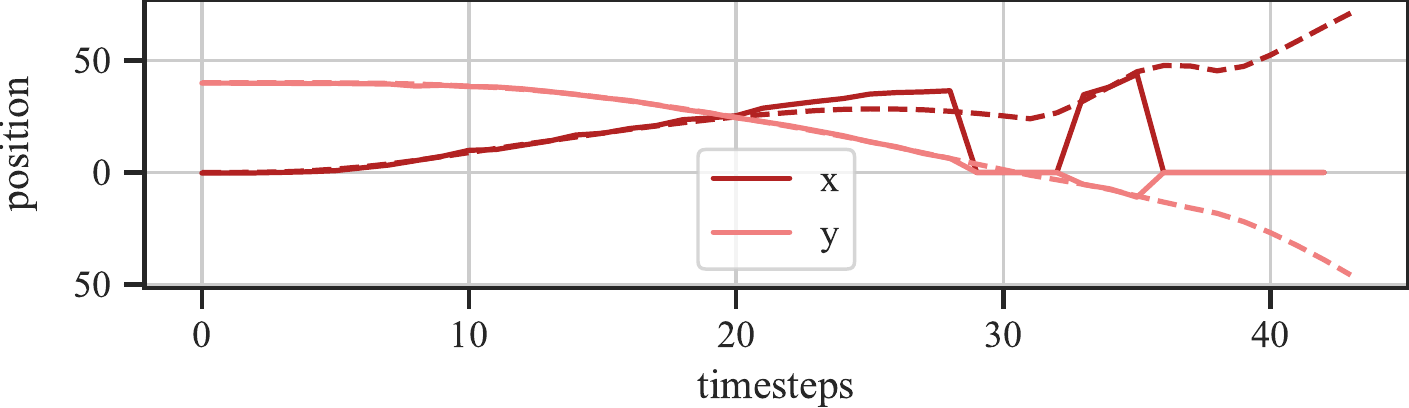}
    \caption{Two different reference trajectories for the UAV with NW perception lead to a success and a failure. Left, reference and actual trajectories. Right, predicted (solid) and actual (dashed) positions for the failed tracking example.}
	\label{fig:trajectories_uav}
\end{figure}

\begin{figure}[t]
    \centering
	\includegraphics[height=7em]{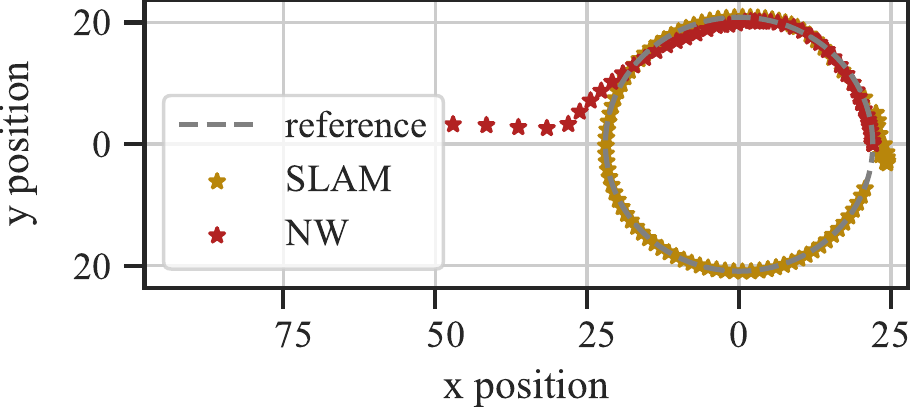}~~~~
	\includegraphics[height=7em]{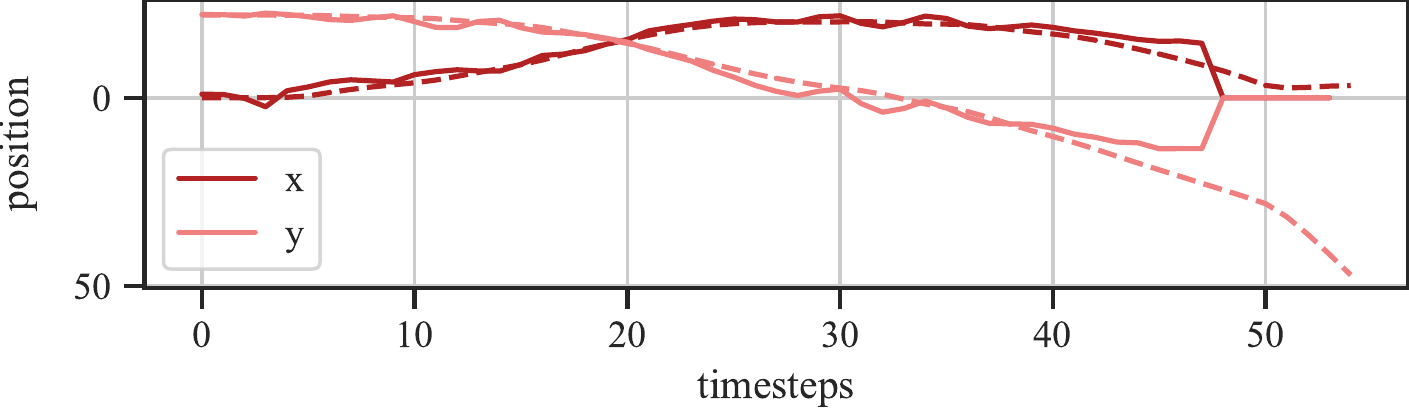}
    \caption{Two different predictors lead to a success and a failure of reference tracking for the car. Left, reference and actual trajectories for NW and SLAM. Right, predicted (solid) and actual (dashed) positions for NW.}
	\label{fig:trajectories}
\end{figure}

\section{Conclusion and Future Work}

We have presented a sample complexity guarantee for the task of using a complex and unmodelled sensor for waypoint tracking.
Our method makes use of noisy measurements from an additional sensor to both learn an inverse {perception} map and collect training data.
We show that evenly sampling the measurement subspace is sufficient for ensuring the success of nonparametric regression, and furthermore that using this learned component in closed-loop has bounded sub-optimality.
Unlike related work that focuses on learning unknown dynamics, the task we consider incorporates both a lack of direct state observation and nonlinearity, making it
relevant to modern robotic systems.

We hope that future work will continue to probe this problem setting 
to rigorously explore relevant trade-offs.
One direction for future work is to contend with the sampling burden by collecting data in a more goal-oriented manner, perhaps with respect to a target task or the continuity of the observation map.
It is also of interest to consider extensions which do not rely on the supervision of a noisy sensor or make clearer use of the structure induced by the dynamics on the observations.
Making closer connections with modern computer vision methods like SLAM could lead to insights about unsupervised and semi-supervised learning, particularly when data has known structure.

\subsubsection*{Acknowledgments and Disclosure of Funding}
We thank Horia Mania, Pavlo Manovi, and Vickie Ye for helpful comments and assistance with experiments.
This research is generously supported in part by ONR awards N00014-17-1-2191, N00014-17-1-2401, and N00014-18-1-2833, NSF CPS award 1931853, and the DARPA Assured Autonomy program (FA8750-18-C-0101).
SD is supported by an NSF Graduate Research Fellowship under Grant No. DGE 1752814

\bibliographystyle{plainnat}
\bibliography{refs}
\newpage
\appendix
\section{Technical Results} \label{app:proofs}

In this section, we provide proofs for the results presented in the body of the paper, as well as some intermediate lemmas.

To begin, recall the pointwise error bound presented in Lemma~\ref{lem:regression_error}.
Its proof relies on the following concentration result, which allows us to handle training data collected from the closed-loop system.
Care is needed because $u_t$ depends on $y_{0:t}$, and a result, $z_t = g(x_t)$ will depend on previous noise variables $\eta_k$ for $k<t$.

\begin{lemma}[adapted from Corollary 1 in~\cite{abbasi2011online}]\label{lem:sum_filtration}
Let $(\mathcal F_t)_{t\geq 0}$ be a filtration, let $(V_t)_{t\geq 0}$ be a real-valued stochastic process adapted to $(\mathcal F_t)$, and let $(W_t)_{t\geq 1}$ be a real-valued martingale difference process adapted to $(\mathcal F_t)$. Assume that $W_t \mid \mathcal F_{t-1}$ is conditionally sub-Gaussian with parameter $\sigma$.
Then for all $T\geq 0$, with probability $1-\delta$,
\[\Big\|\sum_{k=1}^T V_t W_t\Big\|^2_2 \leq 2\sigma^2\log\left(\frac{\sqrt{\sum_{t=1}^T V_t^2}}{\delta}\right)\Big({1+\sum_{t=1}^T V_t^2}\Big)\:.\]
\end{lemma}

\begin{proof}[Proof of Lemma~\ref{lem:regression_error}]
Recall that the error can be bounded by an approximation term and a noise term:
\begin{align*}
\|\widehat h(z) - h(z)\|_\infty
&= \Big\| \sum_{t=0}^T \frac{\ker_\gamma(z_t,z)}{s_T(z)}(Cx_t +\eta_t) -  Cx \Big\|_\infty\\
&\leq \Big\| \sum_{t=0}^T \frac{\ker_\gamma(z_t,z)}{s_T(z)}(Cx_t - Cx) \Big \|_\infty + \Big\| \sum_{t=0}^T \frac{\ker_\gamma(z_t,z)}{s_T(z)} \eta_t \Big\|_\infty\:.
\end{align*}

The approximation error  is bounded by the weighted average value of $\left\|Cx_t - Cx \right \|_\infty$ for points in the training set that are close to the current observation.
Using the continuity of the map $h$ from observations to labels, we have that 
\[\ker_\gamma(z_t,z)>0 \implies \rho(z_t, z) \leq \gamma 
\implies \|Cx_t-Cx\|_\infty\leq \gamma L_h\:. \]
This provides an upper bound on the average.

Turning to the measurement error term, we begin by noting that
\begin{align}
P\left(\Big\|\sum_{t=0}^T \frac{\ker_\gamma(z_t,z)}{s_T(z)} \eta_t\Big\|_\infty\geq s\right) &\leq
\sum_{i=1}^p
P\left( 
 \frac{1}{s_T(z)}\Big|\sum_{t=0}^T\ker_\gamma(z_t,z) \eta_{t,i}\Big| \geq s \right) \:.
\end{align}
We then apply Lemma~\ref{lem:sum_filtration} with $\mathcal F_t = x_{0:t}$, $V_t=\ker_\gamma(z_t,z)$, and $W_t = \eta_{t,i}$. Note that $\eta_{t,i}$ is $\sigma_\eta/2$ sub-Gaussian due to the fact that it is bounded.
Therefore, with probability at least $1-\delta/p^2$,
\begin{align*}
\frac{1}{s_T(z)}\Big|\sum_{t=0}^T\ker_\gamma(z_t,z) \eta_{t,i}\Big|
&\leq \frac{ \sigma_\eta}{2 s_T(z)} \sqrt{2\log\left(\frac{p^2}{\delta}\sqrt{\sum_{k=0}^T \ker_\gamma(z_t,z)^2}\right)({1+\sum_{k=0}^T \ker_\gamma(z_t,z)^2})}\\
&\labelrel\leq{lem1:leq_one}  \frac{\sigma_\eta}{2} \sqrt{2\log\left(\frac{p^2}{\delta}\sqrt{s_T(z)}\right)} \frac{\sqrt{{1+s_T(z)}}}{s_T(z)}
\labelrel\leq{lem1:frac}  \sigma_\eta \sqrt{\frac{\log\left(p^2\sqrt{s_T(z)}/\delta\right)}{s_T(z)}}
\end{align*}
Where \eqref{lem1:leq_one} holds since $\ker_\gamma\leq 1$, and \eqref{lem1:frac} since $s_T(z) \geq 1$ implies that $\frac{\sqrt{{1+s_T(z)}}}{s_T(z)}\leq \sqrt{\frac{2}{s_T(z)}}$.
Then with probability $1-\delta$,
\begin{align}
\Big\|\sum_{t=0}^T \frac{\ker_\gamma(z_t-z)}{s_T(z)} \eta_t\Big\|_\infty 
&\leq \sigma_\eta \sqrt{\frac{\log\left(p^2\sqrt{s_T(z)}/\delta\right)}{s_T(z)}} \:.
\end{align}
\end{proof}

Since Lemma~\ref{lem:regression_error} is only useful for bounding errors for a finite collection of points, we now prove a similar bound that holds over all points.
While it is possible to use the following result in a data-driven way (with $z_i=z_t$ and $H=T$), we will use it primarily for our data independent bound.

\begin{lemma}\label{lem:regression_error_all_pts}
Consider the setting of Lemma~\ref{lem:regression_error}.
Let $\{z_i\}_{i=1}^H$ be any arbitrary set of observations with $s_T(z_i)\neq 0$.
Then for all $z$, the
learned perception map has bounded errors with probability at least $1-\delta$:
\begin{align}
\begin{split}
\|\widehat h(z) - h(z)\|_\infty \leq &\gamma L_h + \min_i \frac{\sigma_\eta}{\sqrt{s_T(z_i)}} \sqrt{\log\left(p^2H\sqrt{s_T(z_i)}/\delta\right)} 
+   \frac{2\sigma_\eta T}{s_T(z_i)}\frac{L_\ker}{\gamma} \rho(z,z_i) 
\:.
\end{split}
\end{align}
\end{lemma}

\begin{proof}

To derive a statement which holds for uncountably many points, we adapt the proof of Lemma~\ref{lem:regression_error}.
Previously, we used the fact that the error can be bounded by an approximation term and a noise term. Now, we additionally consider a smoothness term so that for any $1\leq i\leq H$,
\begin{align*}
\|\widehat h(z) - h(z)\|_\infty
&\leq \Big\| \sum_{t=0}^T \frac{\ker_\gamma(z_t,z)}{s_T(z)}(Cx_t - Cx)\Big\|_\infty +  \Big\|\sum_{t=0}^T \frac{\ker_\gamma(z_t,z_i)}{s_T(z_i)} \eta_t \Big\|_\infty\\ 
&\qquad\qquad\qquad\qquad\qquad\qquad + \Big\|\sum_{t=0}^T \frac{\ker_\gamma(z_t,z)}{s_T(z)} \eta_t- \frac{\ker_\gamma(z_t,z_i)}{s_T(z_i)} \eta_t \Big\|_\infty&\\
&\leq \gamma L_h +  \frac{\sigma_\eta}{\sqrt{s_T(z_i)}} \sqrt{\log\left(p^2H\sqrt{s_T(z_i)}/\delta\right)} + \Big\|\sum_{t=0}^T \frac{\ker_\gamma(z_t,z)}{s_T(z)} \eta_t- \frac{\ker_\gamma(z_t,z_i)}{s_T(z_i)} \eta_t \Big\|_\infty\:.
\end{align*}
The final line holds for all $z_i$ with probability at least $1-H\cdot \delta/H$, following the logic of the proof of Lemma~\ref{lem:regression_error} along with a union bound.
Now, consider the third term and the boundedness of the noise:
\begin{align*}
\Big\|\sum_{t=0}^T \frac{\ker_\gamma(z_t,z)}{s_T(z)} \eta_t- \frac{\ker_\gamma(z_t,z_i)}{s_T(z_i)} \eta_t \Big\|_\infty 
&\leq \sigma_\eta \sum_{t=0}^T \Big|\frac{\ker_\gamma(z_t,z)}{s_T(z)}- \frac{\ker_\gamma(z_t,z_i)}{s_T(z_i)} \Big|\:.
\end{align*}
Using the fact that each term in the expression is nonnegative,
\begin{align*}
\Big|\frac{\ker_\gamma(z_t,z)}{s_T(z)}- \frac{\ker_\gamma(z_t,z_i)}{s_T(z_i)} \Big| &= 
\frac{1}{s_T(z_i)}\Big| \frac{s_T(z_i)}{s_T(z)}\ker_\gamma(z_t,z) - \ker_\gamma(z_t,z_i) \Big| \\
&\leq 
\frac{1}{s_T(z_i)}\left(\left| \ker_\gamma(z_t,z) - \ker_\gamma(z_t,z_i) \right| + \left| \tfrac{s_T(z_i)}{s_T(z)}-1\right | \ker_\gamma(z_t,z)  \right)
\end{align*}
Using the smoothness of the kernel, we have that
\begin{align*}
|\ker_\gamma(z_t,z)-\ker_\gamma(z_t,z_i)| &= 
\Big|\ker\left(\tfrac{\rho(z_t,z)}{\gamma}\right)-\ker\left(\tfrac{\rho(z_t,z_i)}{\gamma}\right)\Big| \leq L_\ker \Big|\tfrac{\rho(z_t,z) - \rho(z_t,z_i)}{\gamma}\Big|\leq \frac{L_\ker}{\gamma} \rho(z,z_i)
\end{align*}
where the final inequality uses the fact that $\rho$ is a metric.
Additionally, we have that
\begin{align*}
\left|1-\tfrac{s_T(z_i)}{s_T(z)}\right| &= \frac{1}{s_T(z)}\left|{s_T(z)}-{s_T(z_i)}\right|\\
& \leq \frac{1}{s_T(z)} \sum_{t=0}^T|\ker_\gamma(z_t,z)-\ker_\gamma(z_t,z_i)| \\
&\leq\frac{1}{s_T(z)} \frac{L_\ker}{\gamma} T \rho(z,z_i)
\end{align*}
Therefore,
\begin{align*}
\sum_{t=0}^T\Big|\frac{\ker_\gamma(z_t,z)}{s_T(z)}- \frac{\ker_\gamma(z_t,z_i)}{s_T(z_i)} \Big| 
&\leq \sum_{t=0}^T
\frac{1}{s_T(z_i)}\left(\frac{L_\ker}{\gamma}  + \frac{1}{s_T(z)} \frac{L_\ker}{\gamma} T\ker_\gamma(z_t,z_i) \right)\rho(z,z) \\
&= 
\frac{1}{s_T(z_i)}\left(T\frac{L_\ker}{\gamma}  + \frac{\sum_{t=0}^T \ker_\gamma(z_t,z) }{s_T(z)} \frac{L_\ker}{\gamma} T \right)\rho(z,z_i) \\
&= 
\frac{2T}{s_T(z_i)}\frac{L_\ker}{\gamma} \rho(z,z_i) 
\end{align*}
The result follows.
\end{proof}

\vspace{1em}
\begin{proof}[Proof of Lemma~\ref{lem:sT_bound_concentration}]
We begin by establishing some properties of the training data generated by Algorithm~\ref{alg:ref_inputs}.
\begin{align*}
h(z_t)=h(z_{n,\ell}) =C x_{n,\ell} &= {\sum_{k=1}^{n} C \Phi_{xu}(k) u^\reff_{n-k, \ell}} + \Big({C \Phi_x(n)x_{0,\ell} + \sum_{k=1}^{n} C \Phi_{xn}(k) \eta_{n-k, \ell}}\Big)\\
& = \begin{bmatrix} C\Phi_{xu}(1) &\dots& C\Phi_{xu}(n)\end{bmatrix} \begin{bmatrix} C\Phi_{xu}(1) &\dots& C\Phi_{xu}(n)\end{bmatrix}^\dagger y^\reff_\ell+{w_\ell}\\
&
= y^\reff_\ell+{w_\ell}
\end{align*}
where in the second line, we use the fact that the pair $(A, B)$ is controllable implies that the matrix is full column rank.
In this expression, $y_\ell^\reff$ is sampled by Algorithm~\ref{alg:ref_inputs} uniformly from $\{y\mid \|y\|_\infty \leq \bar r\}$ and
\[\|w_\ell\|_\infty \leq \|C \Phi_x(n)\|_\infty \|x_0\|_\infty + \sum_{\ell=1}^{n} \|C \Phi_{xn}(\ell)\|_\infty  \|\eta_{n-\ell}\|_\infty \leq M \sigma \sum_{\ell=0}^{n} \rho^\ell \leq \frac{M\sigma}{1-\rho}\:,\]
where we let $\sigma=\max\{\sigma_0,\sigma_\eta\}$

Now, we reason about the density of the training data over the measurement subspace.
Above, we see that $Cx_{n,\ell}$ are generated as the sum of two independent random variables.
The first is a uniform random variable which has density $f_{\mathrm{Unif}}(y) = (2\bar r)^{-p}\mathbf{1}\{\|y\|_\infty \leq \bar r\}$.
The second is a bounded random variable, so
the support of its density $f_w(y)$ is contained within the set $\{\|y\|_\infty \leq  \frac{M\sigma}{1-\rho}\}$.

We now upper and lower bound the density of each $Cx_{n,\ell}$, which we denote as $f$. 
\begin{align*}
f(y) &= \int_{-\infty}^\infty (2\bar r)^{-p}\mathbf{1}\{\|u\|_\infty \leq \bar r\} f_w(u-y)du\\
& = (2\bar r)^{-p}\int_{\|u+y\|_\infty \leq \bar r} f_w(u)du \\
&\leq (2\bar r)^{-p} \:.
\end{align*}
\begin{align*}
f(y) & = (2\bar r)^{-p}\int_{\|u+y\|_\infty \leq \bar r} f_w(u)du \\
&\geq (2\bar r)^{-p} \mathbf{1}\{\|y\|_\infty \leq \bar r - \tfrac{M\sigma}{1-\rho}\} \int_{\|u\|\leq \frac{M\sigma}{1-\rho}} f_w(u)du\\
& = (2\bar r)^{-p} \mathbf{1}\{\|y\|_\infty \leq \bar r - \tfrac{M\sigma}{1-\rho}\}\:.
\end{align*}
The inequality follows by noting that  $\{\|u\|_\infty\leq  \frac{M\sigma}{1-\rho}\} \subseteq \{\|u+y\|_\infty \leq \bar r\}$ whenever $\|y\|_\infty \leq \bar r - \frac{M\sigma}{1-\rho}$.

We now turn our attention to finding a lower bound on $s_T(z) = \sum_{t=1}^T \ker_\gamma(z, g(Cx_t))$.
We will use Bennett's inequality~\citep{bennett1962probability}, which requires computing the expectation and second moment of each $\ker_\gamma(z, g(Cx_t))$.
We begin by lower bounding the expected value.
\begin{align*}
\E_{Cx_t} [\ker_\gamma(z, g(Cx_t))]
&= \int K\left(\tfrac{\rho(g(y), g(u))}{\gamma}\right) f(h(u)) du\\
& \labelrel\geq{lem2:density} (2\bar r)^{-p}\int_{ \{\|u\|_\infty \leq \bar r - \frac{M\sigma}{1-\rho}\}} \ker\left(\tfrac{\rho(z, g(u))}{\gamma}\right) d u\\
& \labelrel\geq{lem2:smoothness} (2\bar r)^{-p}\int_{\{\|u\|_\infty \leq \bar r - \frac{M\sigma}{1-\rho}\}} \ker\left(\tfrac{L_g\|h(z)-u\|_\infty}{\gamma}\right) d u\\
& \labelrel={lem2:inclusion} (2\bar r)^{-p}\int  \ker\left(\tfrac{L_g\|h(z)-u\|_\infty}{\gamma}\right) d u\\
& \labelrel={lem2:changevar} \frac{\gamma^p}{(2\bar rL_g)^p}\int   \ker\left(\|u\|_\infty\right) d u\\
& \labelrel={lem2:Vker}\left(\frac{\gamma }{2\bar rL_g} \right)^p 2^p V_\ker
\end{align*}
where \eqref{lem2:density} follows by the lower bound on $f(y)$, \eqref{lem2:smoothness} follows by the smoothness of $g$, 
\eqref{lem2:inclusion} holds as long as $\|h(z)\|_\infty \leq \bar r - \frac{M\sigma}{1-\rho} - \tfrac{\gamma}{L_g}$ which is guaranteed for $\|h(z)\|_\infty\leq r$ by assumption,
\eqref{lem2:changevar} is a result of a change of variables, and 
\eqref{lem2:Vker} follows by the assumption on $\ker$.
Turning to the second moment,
\begin{align*}
\E_{Cx_t} [\ker_\gamma(z, g(Cx_t))^2]
&= \int \ker\left(\tfrac{\rho(g(y), g(u))}{\gamma}\right)^2 f(h(u)) du\\
& \labelrel\leq{lem2:density_up} 
(2\bar r)^{-p}\int  \ker\left(\tfrac{\rho(z, g(u))}{\gamma}\right)^2 d u\\
& \labelrel\leq{lem2:sq_smooth} 
(2\bar r)^{-p}\int  \ker\left(\tfrac{\|h(z)-u\|_\infty}{L_h \gamma}\right) d u\\
& \labelrel={lem2:change_var_Lh} 
\left(\frac{\gamma L_h}{2\bar r}\right)^{p}\int  \ker\left(\|u\|_\infty \right) d u = \left(\frac{\gamma L_h}{\bar r}\right)^{p}V_\ker
\end{align*}
where~\eqref{lem2:density_up} follows by the upper bound on $f(y)$, ~\eqref{lem2:sq_smooth} follows by the smoothness of $h$ and the fact that $\ker\in [0,1]$, and~\eqref{lem2:change_var_Lh} by a change of variables.

We now bound the deviation
\begin{align*}
P\left( s_T(z) \leq  T\left(\tfrac{\gamma }{\bar rL_g} \right)^p V_\ker   -\eps\right)
&=P\left( \sum_{t=1}^T \left(\left(\tfrac{\gamma }{\bar rL_g} \right)^p V_\ker- \ker_\gamma(z,z_t) \right)\geq    \eps\right)\\
&\labelrel\leq{lem2:Exp_bd} P\left( \sum_{t=1}^T \E[\ker_\gamma(z,z_t)]- \ker_\gamma(z,z_t) \geq    \eps\right)\\
&\labelrel\leq{lem2:Bennet}
e^{-\eps^2/(2T(\frac{\gamma L_h}{\bar r})^{p}V_\ker)}
\end{align*}
where~\eqref{lem2:Exp_bd} follows by our lower bound on the expectation and~\eqref{lem2:Bennet} follows by a one-side version of Bennet's inequality for nonnegative random variables, which uses the fact that $\ker_\gamma(z,z_t) \perp \ker_\gamma(z,z_k)$ for $k\neq t$.

Therefore, with probability at least $1-\delta$,
\begin{align*} 
 s_T(z) &\geq T V_\ker \left(\frac{\gamma}{\bar rL_g} \right)^p - \sqrt{2V_\ker T\log(1/\delta)}\left(\frac{\gamma L_h}{\bar r}\right)^{\frac{p}{2}}\\
 &\geq \frac{1}{2}\sqrt{T V_\ker} \left(\frac{\gamma}{\bar r L_g } \right)^{\frac{p}{2}}
\end{align*}
where the second inequality holds
under the assumption that
$T  \geq 8V_\ker^{-1}\log(1/\delta) \left(\frac{\bar rL_h L_g^2}{\gamma} \right)^{p}$.
\end{proof}

\vspace{1em}
We are now ready to show the main result by combining Lemma~\ref{lem:regression_error_all_pts} with Lemma~\ref{lem:sT_bound_concentration}.

\begin{proof}[Proof of Theorem~\ref{thm:uniform_convergence}]
We begin by specifying the set of $\{z_i\}$ to be used in the application of Lemma~\ref{lem:regression_error_all_pts}.
By selecting $z_i = g(y_i)$ for $\{y_i\}_{i=1}^H$ defined as an $\eps$ grid of $\{y\mid\|y\|_\infty\leq r\}$, we have that for any $z$, $\min_i\rho(z,z_i) \leq \min_iL_h\|h(z)-y_i\|_\infty \leq L_h\eps$. 
Then notice that an $\eps$ grid requires $H \geq \left(2r/\eps\right)^p$ points, or rearranging, that $\eps\leq 2rH^{-1/p}$.
Therefore, we have that with probability at least $1-\delta/2$
\begin{align}
\begin{split}
\label{eq:lem_err_bd}
\|\widehat h(z) - h(z)\|_\infty \leq &\gamma L_h +  \frac{\sigma_\eta}{\sqrt{ s_T(z_{i_\star})}} \sqrt{\log\left(2p^2H\sqrt{T}/\delta\right)} 
+ 4\sigma_\eta  \frac{T}{ s_T(z_{i_\star})}\frac{L_\ker}{\gamma} L_h rH^{-1/p}
\:,
\end{split}
\end{align}
where $i_\star$ is the index of the closest element of $\{z_i\}_{i=1}^H$ to $z$.

Next, we use Lemma~\ref{lem:sT_bound_concentration} to show that each $s_T(z_{i_\star})$ is bounded below.
First, notice that for $\bar r = \sqrt{2}r$ the condition in Lemma~\ref{lem:sT_bound_concentration} is met by the assumption that
$\gamma\leq L_g((\sqrt{2}-1)r - \frac{M\sigma}{1-\rho})$.
Therefore for any $1\leq i\leq H$, with probability at least $1-H\cdot \delta/(2H)$, as long as $T  \geq 8V_\ker^{-1}\log(2H/\delta) \left(\frac{\bar rL_h L_g^2}{\gamma} \right)^{p}$,
\begin{align}
\label{eq:lem_sT_bd}
s_T(z_i) \geq \frac{1}{2}\sqrt{T V_\ker} \left(\frac{\gamma}{\bar r L_g } \right)^{\frac{p}{2}}=: \underline s_T \gamma^{\frac{p}{2}}\:.
\end{align}
We can therefore select $H$ to balance terms, with
\begin{align}
\begin{split}
\   H^{1/p}
=\frac{4 T L_\ker L_h r}{\sqrt{\underline s_T}\gamma^{\frac{p+4}{4}}} =\frac{4 \sqrt{2}T^{3/4} L_\ker L_h r^{\frac{p+4}{4}}  L_g^{\frac{p}{4}}}{ V_\ker^{1/4} \gamma^{\frac{p+4}{4}}} \:.
\end{split}
\end{align}
For this choice of $H$, we have by union bound that with probability at least $1-\delta$,
\begin{align}
\begin{split}\label{eq:eps_h_bound_unsimplified}
\|\widehat h(z) - h(z)\|_\infty \leq &\gamma L_h +  \frac{\sigma_\eta}{\sqrt{\underline s_T}\gamma^{p/4}} \left(\sqrt{\log\left(2p^2H \sqrt{T}/\delta\right)} +1\right)
\:.
\end{split}
\end{align}
Bounding the logarithmic term,
\begin{align*}
{\log\left(2p^2H \sqrt{T}/\delta\right)} &= 
p  {\log\left(2p^{2/p} \left(
\frac{4 \sqrt{2}T^{3/4} L_\ker L_h r^{\frac{p+4}{4}}  L_g^{\frac{p}{4}}}{ V_\ker^{1/4} \gamma^{\frac{p+4}{4}}}
\right)  
T^{1/2p}/\delta^{1/p}\right)}
\\
&\leq 
{p}  {\log\left(24  V_\ker^{-\frac{1}{4}} L_\ker L_h L_g ^{\frac{p}{4}}  \left(\tfrac{ r}{\gamma}\right)^{\frac{p+4}{4}}  T^{\frac{3p+2}{4p}} /\delta^{\frac{1}{p}} \right)}  \\
&\leq 
{p} {\log\left( T^2/\delta  \right)} 
\:.
\end{align*}
where the first inequality follows because $p^{2/p}\leq 3/\sqrt{2}$ and the final inequality follows due to the assumption that 
$(24   L_\ker L_h)^{\frac{4}{3}} V_\ker^{-\frac{1}{3}} \left(\tfrac{r}{\gamma}\right)^{\frac{p+4}{3}}  L_g ^{\frac{p}{3}} \leq T$.

The result follows. It only remains to note that the condition on $T$ in Lemma~\ref{lem:sT_bound_concentration} is satisfied due to the assumption that
 \[T  \geq  8V_\ker^{-1}p\log(  T^2/\delta ) \left(\frac{\bar rL_h L_g^2}{\gamma} \right)^{p}\]
by noting that $\log(2H/\delta)\leq 
{\log(2p^2H \sqrt{T}/\delta)}
 \leq 
 {p\log(  T^2/\delta )}$.
\end{proof}

\ifpreprint
\vspace{1em}
\fi

\begin{proof}[Proof of Corollary~\ref{coro:main_result}]
To show this result, we combine the expressions in Theorem~\ref{thm:uniform_convergence} and Proposition~\ref{prop:subopt} for a carefully chosen value of $\gamma$.
To balance the terms, we set (recalling that $\bar r = \sqrt{2}r = 2r_{\max}(\tf \Phi)$)
\[\gamma^{(p+4)/4} 
= \frac{p^{1/2}\sigma_\eta}{ L_h\sqrt{\underline s_T}} 
= 
\frac{\sqrt{2} p^{1/2}\sigma_\eta (2r_{\max} L_g)^{\frac{p}{4}}}{ L_h (T V_\ker)^{\frac{1}{4}}
} 
\] 
Returning to the unsimplified bound~\eqref{eq:eps_h_bound_unsimplified}, this choice of $\gamma$ results in
\begin{align}
\begin{split}
\|\widehat h(z) - h(z)\|_\infty \leq & (2r_{\max}(\tf \Phi) L_gL_h)^{\frac{p}{p+4}} \left(\frac{4p^2\sigma_\eta^4}{ T }\right)^{\frac{1}{p+4}}  \left(\sqrt{\frac{1}{p}\log\left(2p^2H \sqrt{T}/\delta\right)} +\frac{1}{\sqrt{p}} + 1\right)
\:.
\end{split}
\end{align}
We begin by simplifying the first term. Note that since $h\circ g$ is an identity map,  $L_hL_g \geq 1$, so
$(2r_{\max}(\tf \Phi) L_gL_h)^{\frac{p}{p+4}}\leq 2r_{\max}(\tf \Phi) L_gL_h$ since $r_{\max}(\tf\Phi)\geq 1/2$ by assumption.

Next, considering the the logarithmic term:
\begin{align*}
{\log\left(2p^2H \sqrt{T}/\delta\right)} &= 
  p{\log
  \left(2p^{2/p} 
  \left(\frac{4T L_\ker L_h (2r_{\max}(\tf \Phi))}{\sqrt{\underline s_T}\frac{p^{1/2}\sigma_\eta}{ L_h\sqrt{\underline s_T}} 
  } \right)
    T^{1/2p}/\delta^{1/p}\right)}\\
&= p{\log\left(2p^{2/p-1/2} \left(\frac{8L_\ker L_h^2 r_{\max}(\tf \Phi) }{\sigma_\eta }\right)  T^{1+1/2p}/\delta^{1/p}\right)}\\
&\leq 
 p{\log\left( 3\cdot 8 L_\ker L_h^2 r_{\max}(\tf \Phi) \sigma_\eta ^{-1}  T^{3/2}/\delta^{1/p}\right)}\\
 &\leq 
 p{\log\left(  T^2/\delta \right)}
\end{align*}
where the first inequality follows because $p^{2/p-1/2}\leq 1.5$ and the final inequality is true as long as $(24 L_\ker L_h^2 r_{\max}(\tf \Phi) \sigma_\eta^{-1})^2 \leq T$.
Further simplifying,
\begin{align*}
\sqrt{\frac{1}{p}\log\left(2p^2H \sqrt{T}/\delta\right)} +\frac{1}{\sqrt{p}} + 1 \leq
\sqrt{\log\left(  T^2/\delta \right)} +\frac{1}{\sqrt{p}} + 1
 \leq 2\sqrt{\log\left(  T^2/\delta \right)}\:.
\end{align*}
The resulting error bound is
\begin{align}\label{eq:epsh_def}
\|\widehat h(z) - h(z)\|_\infty\leq  4L_gL_h r_{\max}(\tf\Phi)  \left(\frac{4p^2\sigma^4 }{T}\right)^{\frac{1}{p+4}} \sqrt{\log(T^2/\delta)} =: \varepsilon_h\:.
\end{align}

Then the result follows by Proposition~\ref{prop:subopt} as long as
 $T\geq \max\{T_1, T_2, T_3, T_4\}$,
where
we ensure that the simplification in the logarithm above is valid with 
\[T_1 = (24 L_\ker L_h^2 r_{\max}(\tf \Phi) \sigma_\eta^{-1})^2, \]
we ensure that $\gamma\leq L_g((2-\sqrt{2})r_{\max}(\tf\Phi) - \frac{M\sigma}{1-\rho}) $ (Theorem~\ref{thm:uniform_convergence}) with
\[T_2 = \frac{(4p\sigma_\eta^2)^2 (2r_{\max}(\tf\Phi))^p L_g^4}{V_\ker ((2-\sqrt{2})r_{\max}(\tf\Phi) - M\frac{\sigma}{1-\rho})^{p+4}L_h^4},\]
we ensure 
	$T  \geq 8V_\ker^{-1}\log(2H/\delta) \left(\tfrac{\bar rL_h L_g^2}{\gamma} \right)^{p}$ (Theorem~\ref{thm:uniform_convergence}'s usage of Lemma~\ref{lem:sT_bound_concentration}) with
	\[T_3 = V_\ker \left(8(L_g L_h)^p{p\log(  T^2/\delta )}\right)^{\frac{p+4}{4}} 
	\left(
	\frac{ L_g L_h r_{\max}(\tf \Phi) }
	{\sqrt{2} p^{1/2}\sigma_\eta }
	\right)^p,\]
and finally we ensure that  $\eps_h\leq \frac{(\sqrt{2}-1) r_{\max}(\tf\Phi)}{\|C\tfxn \|_\lone}$ (Proposition~\ref{prop:subopt}) with
\begin{align*}
   T_4 = 4p^2\sigma_\eta^4 \left(10L_gL_h\|C\tfxn \|_\lone\sqrt{\log(T^2/\delta)}\right)^{p+4 }\:.
\end{align*}
For large enough values of $T$, $T_4$ will dominate.

\end{proof}
 
\end{document}